\documentclass{article}
\usepackage[left=1.1 in, right=1.1 in,top=1 in, bottom=1 in]{geometry}
\usepackage[pdftex]{graphicx}
\usepackage{amsfonts}
\usepackage{latexsym}
\usepackage{tabularx}
\usepackage{fancyhdr}
\usepackage{verbatim}
\usepackage{multirow}
\usepackage{framed}
\usepackage{algorithmic}
\usepackage{algorithm}
\usepackage{amsmath,amsthm,amssymb}
\usepackage[pdftex,colorlinks,citecolor=blue,linkcolor=red]{hyperref}

\usepackage{titling} 
\usepackage{graphicx}
\usepackage{caption}
\usepackage{subcaption}
\usepackage{enumitem}
\usepackage{lscape}
\usepackage{url}
\usepackage[]{authblk}

\allowdisplaybreaks[4] 

\hypersetup{pdfauthor={Y.W. Park and D. Klabjan}} \hypersetup{pdftitle= } \hypersetup{pdfsubject=}

\theoremstyle{definition}

\newtheorem{lemma}{Lemma}
\newtheorem{assumption}{Assumption}

\newtheorem{proposition}{Proposition}
\newtheorem{property}{Property}
\theoremstyle{definition} \newtheorem*{example}{Example}
\theoremstyle{definition}

\makeatletter
\newcommand{\rmnum}[1]{\romannumeral #1}
\newcommand{\Rmnum}[1]{\expandafter\@slowromancap\romannumeral #1@}
\makeatother

\title{Bayesian Network Learning via Topological Order} 

\author[1]{Young Woong Park \thanks{ywpark@iastate.edu}}
\author[2]{Diego Klabjan \thanks{d-klabjan@northwestern.edu}}
\affil[1]{College of Business, Iowa State University, Ames, IA 50011, USA}
\affil[2]{Department of Industrial Engineering and Management Sciences, Northwestern University, Evanston, IL 60208, USA}

\date{\today}

\begin{document}

\maketitle

\begin{abstract}
We propose a mixed integer programming (MIP) model and iterative algorithms based on topological orders to solve optimization problems with acyclic constraints on a directed graph. The proposed MIP model has a significantly lower number of constraints compared to popular MIP models based on cycle elimination constraints and triangular inequalities. The proposed iterative algorithms use gradient descent and iterative reordering approaches, respectively, for searching topological orders. A computational experiment is presented for the Gaussian Bayesian network learning problem, an optimization problem minimizing the sum of squared errors of regression models with L1 penalty over a feature network with application of gene network inference in bioinformatics.
\end{abstract}

\section{Introduction}
\label{section_intro}

Directed graph $G$ is a directed acyclic graph (DAG) or acyclic digraph if $G$ does not contain a directed cycle. In this paper, we consider a generic optimization problem over a directed graph with acyclic constraints, which require the selected subgraph to be a DAG. 

Let us consider a complete digraph $G$. Let $m$ be the number of nodes in digraph $G$, $Y \in \mathbb{R}^{m \times m}$ a decision variable matrix associated with the arcs, where $Y_{jk}$ is related to arc $(j,k)$, $supp(Y) \in \{0,1\}^{m \times m}$ the 0-1 (adjacency) matrix with $supp(Y)_{jk} = 1$ if $Y_{jk} \neq 0$, $supp(Y)_{jk} = 0$ otherwise, $G(supp(Y))$ the sub-graph of $G$ defined by $supp(Y)$, and let $\mathcal{A}$ be the collection of all acyclic subgraphs of $G$. Then, we can write the optimization problem with acyclic constraints as
\begin{equation}
\label{def_opt_with_dag}
\displaystyle \min_{Y} \mbox{ } F(Y) \quad s.t. \quad G(supp(Y)) \in \mathcal{A},
\end{equation}
where $F$ is a function of $Y$.

Acyclic constraints (or DAG constraints) appear in many network structured problems. The maximum acyclic subgraph problem (MAS) is to find a subgraph of $G$ with maximum cardinality while the subgraph satisfies acyclic constraints. MAS can be written in the form of \eqref{def_opt_with_dag} with $F(Y) = - \| supp(Y)\|_0$. Although exact algorithms were proposed for a superclass of cubic graphs \cite{fernau2008exact} and for general directed graphs \cite{kaas1981branch}, most of the works have focused on approximations \cite{even1998approximating,hassin1994approximations} or inapproximability \cite{guruswami2008beating} of either MAS or the minimum feedback arc set problem (FAS). FAS of a directed graph $G$ is a subgraph of $G$ that creates a DAG when the arcs in the feedback arc set are removed from $G$. Note that MAS is closely related to FAS and is dual to the minimum FAS. Finding a feedback arc set with minimum cardinality is $\mathcal{NP}$-complete in general \cite{karp1972reducibility}. However, minimum FAS is solvable in polynomial time for some special graphs such as planar graphs \cite{lucchesi1978minimax} and reducible flow graphs \cite{ramachandran1988finding}, and a polynomial time approximation scheme was developed for a special case of minimum FAS, where exactly one arc exists between any two nodes (called tournament) \cite{kenyon2007rank}.

DAGs are also extensively studied in Bayesian network learning. Given observational data with $m$ features, the goal is to find the true unknown underlying network of the nodes (features) while the selected arcs (dependency relationship between features) do not create a cycle. In the literature, approaches are classified into three categories: (\rmnum{1}) score-based approaches that try to optimize a score function defined to measure fitness, (\rmnum{2}) constraint-based approaches that test conditional independence to check existence of arcs between nodes (\rmnum{3}) and hybrid approaches that use both constraint and score-based approaches. Although there are many approaches based on the constraint-based or hybrid approaches, our focus is solving \eqref{def_opt_with_dag} by means of score-based approaches. For a detailed discussion of constraint-based and hybrid approaches and models for undirected graphs, the reader is referred to Aragam and Zhou \cite{aragam2015concave} and Han \textit{et al.} \cite{han2016estimation}.

For estimating the true network structure by a score-based approach, various functions have been used as different functions give different solutions and behave differently. Many works focus on penalized least squares, where penalty is used to obtain sparse solutions. Popular choices of the penalty term include BIC \cite{lam1994learning}, $L_0$-penalty \cite{chickering2002optimal, van2013L0}, $L_1$-penalty \cite{han2016estimation}, and concave penalty \cite{aragam2015concave}. Lam and Bacchus \cite{lam1994learning} use minimum-description length as a score function, which is equivalent to BIC. Chickering \cite{chickering2002optimal} proposes a two-phase greedy algorithm, called greedy equivalence search, with the $L_1$ norm penalty. Van de Geer and B\"uhlmann \cite{van2013L0} study the properties of the $L_0$ norm penalty and show positive aspects of using $L_0$ regularization. Raskutti and Uhler \cite{raskutti2013learning} use a variant of the $L_0$ norm. They use cardinality of the selected subgraph as the score function where the subgraphs not satisfying the Markov assumption are penalized with a very large penalty. Aragam and Zhou \cite{aragam2015concave} introduce a generalized penalty, which includes the concave penalty, and develop a coordinate descent algorithm. Han \textit{et al.} \cite{han2016estimation} use the $L_1$ norm penalty and propose a Tabu search based greedy algorithm for reduced arc sets by neighborhood selection in the pre-processing step.

With any choice of a score function, optimizing the score function is computationally challenging, because the number of possible DAGs of $G$ grows super exponentially in the number of nodes $m$ and learning Bayesian networks is also shown to be $\mathcal{NP}$-complete \cite{chickering1996learning}. Many heuristic algorithms have been developed based on greedy hill climbing \cite{chickering2002optimal, heckerman1995learning, han2016estimation} or coordinate descent \cite{fu2013learning}, or enumeration \cite{raskutti2013learning} when the score function itself is the main focus. There also exist exact solution approaches based on mathematical programming. One of the natural approaches is based on cycle prevention constraints, which are reviewed in Section \ref{section_mip_formulation}. The model is covered in Han \textit{et al.} \cite{han2016estimation} as a benchmark for their algorithm, but the MIP based approach does not scale; computational time increases drastically as data size increases and the underlying algorithm cannot solve larger instances. Baharev \textit{et al.} \cite{baharev2015} studied MIP models for minimum FAS based on triangle inequalities and set covering models. Several works have been focused on the polyhedral study of the acyclic subgraph polytopes \cite{bolotashvili1999new, goemans1996strongest, grotschel1985acyclic, leung1994more}. In general, MIP models have gotten relatively less attention due to the scalability issue.

In this paper, we propose an MIP model and iterative algorithms based on the following well-known property of DAGs \cite{Cook:1998:CO:273025}.
\begin{property}
\label{property_DAG_topological_order}
A directed graph is a DAG if and only if it has a topological order.
\end{property}
A \textit{topological order} or \textit{topological sort} of a DAG is a linear ordering of all of the nodes in the graph such that the graph contains arc $(u,v)$ if and only if $u$ appears before $v$ in the order \cite{introAlgo}. Suppose that $Z$ is the adjacency matrix of an acyclic graph. Then, by sorting the nodes of acyclic graph $G(Z)$ based on the topological order, we can create a lower triangular matrix from $Z$, where row and column indices of the lower triangular matrix are in the topological order. Then, any arc in the lower triangular matrix can be used without creating a cycle. By considering all arcs in the lower triangular matrix, we can optimize $F$ in \eqref{def_opt_with_dag} without worrying to create a cycle. This is an advantage compared to arc-based search, where acyclicity needs to be examined whenever an arc is added. Although the search space of topological orders is very large, a smart search strategy for a topological order may lead to a better algorithm than the existing arc-based search methods. Node orderings are used for Bayesian Network learnings based on Markov chain Monte Carlo methods \cite{Ellis08,Friedman2003,niinimaki2016structure} as alternatives to network structure based approaches.

The proposed MIP assigns node orders to all nodes and add constraints to satisfy Property \ref{property_DAG_topological_order}. The iterative algorithms search over the topological order space by moving from one topological order to another order. The first algorithm uses the gradient to find a better topological order and the second algorithm uses historical choice of arcs to define the score of the nodes.

With the proposed MIP model and algorithms for \eqref{def_opt_with_dag}, we consider a Gaussian Bayesian network learning problem with $L_1$ penalty for sparsity, which is discussed in detail in Section \ref{section_reg_network}. Out of many possible models in the literature, we pick the $L_1$-penalized least square model from recently published work of Han \textit{et al.} \cite{han2016estimation}, which solves the problem using a Tabu search based greedy algorithm. The algorithm is one of the latest algorithms based on arc search and is shown to be scalable when $m$ is large. Further, their score function, $L_1$ penalized least squares, is convex and can be solved by standard mathematical optimization packages. Hence, we select the score function from Han \textit{et al.} \cite{han2016estimation} and use their algorithm as benchmark. In the computational experiment, we compare the performance of the proposed MIP model and algorithms against the algorithm in Han \textit{et al.} \cite{han2016estimation} and other available MIP models for synthetic and real instances.

Our contributions are summarized in the following.
\begin{enumerate}
\item We consider a general optimization problem with acyclic constraints and propose an MIP model and iterative algorithms for the problem based on the notion of topological orders.
\item The proposed MIP model has significantly less constraints than the other MIP models in the literature, while maintaining the same order of the number of variables. The computational experiment shows that the proposed MIP model outperforms the other MIP models when the subgraph is sparse.
\item The iterative algorithms based on topological orders outperform when the subgraph is dense. They are more scalable than the benchmark algorithm of Han \textit{et al.} \cite{han2016estimation} when the subgraph is dense.
\end{enumerate}

In Section \ref{section_mip_formulation}, we present the new MIP model along with two MIP models in the literature. In Section \ref{section_fast_algorithm}, we present two iterative algorithms based on different search strategies for topological orders. The Gaussian Bayesian network learning problem with $L_1$-penalized least square is introduced and computational experiment are presented in Sections \ref{section_reg_network} and \ref{sec_experiment}, respectively.

In the rest of the paper, we use the following notation.
\begin{enumerate}[noitemsep]
\item[] $J = \{1,\cdots,m\}$ = index set of the nodes
\item[] $J^k = J \setminus \{k\}$ = index set of the nodes excluding node $k$, $k \in J$
\item[] $Z = supp(Y) \in \{0,1\}^{m \times m}$
\item[] $\pi = $ topological order
\end{enumerate}
Given $\pi$, we define $\pi_k = q$ to denote that the order of node $k$ is $q$. For example, given three nodes $a,b,c$, and topological order $b-c-a$, we have $\pi_a = 3$, $\pi_b = 1$, and $\pi_c = 2$. With this notation, if $\pi_j > \pi_k$, then we can add an arc from $j$ to $k$.

\section{MIP Formulations based on Topological Order}
\label{section_mip_formulation}

In this section, we present three MIP models for \eqref{def_opt_with_dag}. The first and second models, denoted as \textsf{MIPcp} and \textsf{MIPin}, respectively, are models in the literature for similar problems with acyclic constraints. The third model, denoted as \textsf{MIPto}, is the new model we propose based on Property \ref{property_DAG_topological_order}.

A popular mathematical programming based approach for solving \eqref{def_opt_with_dag} is the cutting plane algorithm, which is well-known for the traveling salesman problem formulation. Let $\mathcal{C}$ be the set of all possible cycles and $C_l \in \mathcal{C}$ be the set of the arcs defining a cycle. Let $H(supp(Y), C_l)$ be a function that counts the number of selected arcs in $supp(Y)$ from $C_l$. Then, \eqref{def_opt_with_dag} can be solved by
\begin{equation}
\label{def_opt_with_dag_cuttingplane}
\mbox{\textbf{\textsf{MIPcp}}} \quad \displaystyle \min_{Y} \mbox{ } F(Y) \quad s.t. \quad H(supp(Y),C_l) \leq |C_l|-1 , C_l \in \mathcal{C},
\end{equation}
which can be formulated as an MIP. Note that \eqref{def_opt_with_dag_cuttingplane} has exponentially many constraints due to the cardinality of $\mathcal{C}$. Therefore, it is not practical to pass all cycles in $\mathcal{C}$ to a solver. Instead, the cutting plane algorithm starts with an empty active cycle set $\mathcal{C}^A$ and iteratively adds cycles to $\mathcal{C}^A$. That is, the algorithm iteratively solves
\begin{equation}
\label{def_opt_with_dag_cuttingplane_iter}
\displaystyle \min_{Y} \mbox{ } F(Y) \quad s.t. \quad H(supp(Y),C_l) \leq |C_l|-1 , C_l \in \mathcal{C}^A,
\end{equation}
with the current active set $\mathcal{C}^A$, detects cycles from the solution, and adds the cycles to $\mathcal{C}^A$. The algorithm terminates when there is no cycle detected from the solution of \eqref{def_opt_with_dag_cuttingplane_iter}. One of the drawbacks of the cutting plane algorithm based on \eqref{def_opt_with_dag_cuttingplane_iter} is that in the worst case we can add all exponentially many constraints. In fact, Han \textit{et al.} \cite{han2016estimation} study the same model and concluded that the cutting plane algorithm does not scale.

Baharev \textit{et al.} \cite{baharev2015} recently presented MIP models for the minimum feedback arc set problem based on linear ordering and triangular inequalities, where the acyclic constraints presented were previously used for cutting plane algorithms for the linear ordering problem \cite{grotschel1984cutting,mitchell2000solving}. For any $F$, we can write the following MIP model based on triangular inequalities presented in \cite{baharev2015,grotschel1984cutting,mitchell2000solving}.
\begin{subequations}
\label{mip_acyclic_baharev_triangle}
\begin{align}
\mbox{\textbf{\textsf{MIPin}}} \quad \min \quad & F(Y)\\
s.t. \quad & Z=supp(Y), \\
& Z_{qj} + Z_{jk} - Z_{qk} \leq 1 , & & 1 \leq q < j <k \leq m, \label{mip_acyclic_baharev_triangle_c}\\
& -Z_{qj} - Z_{jk} + Z_{qk} \leq 0 , & & 1 \leq q < j <k \leq m, \\
& Z_{jk} \in \{0,1\}, & & 1 \leq j <k \leq m \label{mip_acyclic_baharev_triangle_e}
\end{align}
\end{subequations}
Note that $Z_{jk}$ is not defined for all $j \in J^k$ and $k \in J$. Instead of having a full matrix of binary variables, the formulation only uses lower triangle of the matrix using the fact that $Z_{jk} + Z_{kj} = 1$. We can also use this technique to any of the MIP models presented in this paper. However, for ease of explanation, we will use the full matrix, while the computational experiment is done with the reduced number of binary variables. Therefore, the cutting plane algorithm with \textsf{MIPcp} should be more scalable than the implementation in Han \textit{et al.} \cite{han2016estimation}, which has twice more binary variables.

Baharev \textit{et al.} \cite{baharev2015} also provides a set covering based MIP formulation. The idea is similar to \textsf{MIPcp}. In the set covering formulation, each row and column represents a cycle and an arc, respectively. Similar to \textsf{MIPcp}, existence of exponentially many cycles is a drawback of the formulation and Baharev \textit{et al.} \cite{baharev2015} use the cutting plane algorithm.

Next, we propose an MIP model based on Property \ref{property_DAG_topological_order}. Although \textsf{MIPin} uses significantly less constraints than \textsf{MIPcp}, \textsf{MIPin} still has $O(m^3)$ constraints which grows rapidly in $m$. On the other hand, the MIP model we propose has $O(m^2)$ variables and $O(m^2)$ constraints. In addition to $Z$, let us define decision variable matrix $O \in \{0,1\}^{m \times m}$.

\begin{itemize}
\item[] \begin{tabular}{ll}
$O_{kq}= \left \{
	\begin{array}{ll}
		1  &  $ if $\pi_k = q$,  $\\
		0 & $ otherwise, $
	\end{array}
\right.$ & $k \in J,q \in J$
\end{tabular}
\end{itemize}

Then, we have the following MIP model.

\begin{subequations}
\label{def_acyclic_dag_compact}
\begin{align}
\mbox{\textbf{\textsf{MIPto}}} \quad \min \quad & F(Y)\\
s.t. \quad & Z=supp(Y), \label{def_acyclic_dag2_constraint_b}\\
& Z_{jk} - m Z_{kj} \leq \sum_{r \in J} r ( O_{kr} - O_{jr}), && j \in J^k, k \in J, \label{DAG_constraints_new_1}\\
& Z_{jk} + Z_{kj} \leq 1, && j \in J^k, k \in J, \label{DAG_constraints_new_2}\\
& \sum_{q \in J^k} O_{kq} = 1, & & k \in J,\label{def_acyclic_dag2_constraint_d} \\
& \sum_{k \in J^q} O_{kq} = 1, & & q \in J, \label{def_acyclic_dag2_constraint_e}\\
& Z, O \in \{0,1\}^{m \times m}, Y \mbox{ unrestricted} \label{def_acyclic_dag2_constraint_f}
\end{align}
\end{subequations}
The key constraint in \eqref{def_acyclic_dag_compact} is \eqref{DAG_constraints_new_1}. Recall that $Z_{jk}$ indicates which node comes first in the topological order and $O_{kr}$ stores the exact location in the order. With these definitions, \eqref{DAG_constraints_new_1} forces correct values of $Z_{jk}$ and $Z_{kj}$ by comparing the order difference. Recall that we can reduce the number of binary variables $Z_{jk}$'s by plugging $Z_{jk} + Z_{kj} = 1$, but we keep the full matrix notation for ease of explanation. We next show that \eqref{def_acyclic_dag_compact} correctly solves \eqref{def_opt_with_dag}.

\begin{proposition}
An optimal solution to \eqref{def_acyclic_dag_compact} is an optimal solution to \eqref{def_opt_with_dag}.
\end{proposition}
\begin{proof}
By Property \ref{property_DAG_topological_order}, any DAG has a corresponding topological order. Let $\pi^*$ be the topological order defined by an optimal solution $Y^*$ for \eqref{def_opt_with_dag}. Note that \eqref{def_acyclic_dag2_constraint_d} and \eqref{def_acyclic_dag2_constraint_e} define a topological order. Hence, it suffices to show that \eqref{def_acyclic_dag_compact} gives a DAG given $\pi^*$. Note that the right hand side of \eqref{DAG_constraints_new_1} measures the difference in the topological order between nodes $j$ and $k$. If the value is positive, it implies $\pi_k > \pi_j$. Consider \eqref{DAG_constraints_new_1} for $j_1$ and $j_2$ with $\pi_{j_2}^* > \pi_{j_1}^*$. When $j = j_1$ and $k=j_2$, we have $\sum_{r \in J} r ( O_{j_2 r} - O_{j_1 r}) >0$ in \eqref{DAG_constraints_new_1} and at most one of $z_{j_1 j_2}$ and $z_{j_2 j_1}$ can be 1 by \eqref{DAG_constraints_new_2}. When $j = j_2$ and $k=j_1$, we have $\sum_{r \in J} r ( O_{j_1 r} - O_{j_2 r}) <0$ in \eqref{DAG_constraints_new_1} and we must have $z_{j_1 j_2} = 1$ by the left hand side of \eqref{DAG_constraints_new_1}. Therefore, we have correct value $z_{j_1 j_2} = 1$ when $\pi_{j_2}^* > \pi_{j_1}^*$. This completes the proof.
\end{proof}

In Table \ref{tab:compareMIP}, we compare the MIP models introduced in this section. Although all three MIP models have $O(m^2)$ binary variables, \textsf{MIPto} has more binary variables than \textsf{MIPcp} and \textsf{MIPin} due to $O_{kq}$'s. MIP models \textsf{MIPin} and \textsf{MIPto} have polynomially many constraints, whereas  \textsf{MIPcp} has exponentially many constraints. \textsf{MIPto} has the smallest number of constraints among the three MIP models. In the computational experiment, we use a variation of the cutting plane algorithm for \textsf{MIPcp} as it has exponentially many constraints. For \textsf{MIPin} and \textsf{MIPto}, we do not use a cutting plane algorithm.

\begin{table}[htbp]
  \centering
    \begin{tabular}{|l|l|l|l|}
    \hline
    Name     & Reference & \# binary variables & \# constraints \\ \hline
    \textsf{MIPcp}  & \eqref{def_opt_with_dag_cuttingplane} & $O(m^2)$ & exponential \\
    \textsf{MIPin} & \eqref{mip_acyclic_baharev_triangle} & $O(m^2)$ & $O(m^3)$\\
    \textsf{MIPto}& \eqref{def_acyclic_dag_compact}  & $O(m^2)$ & $O(m^2)$ \\ \hline
    \end{tabular}%
      \caption{Number of binary variables and constraints of MIP models}
  \label{tab:compareMIP}%
\end{table}%

\section{Algorithms based on Topological Order}
\label{section_fast_algorithm}

Although the MIP models introduced in Section \ref{section_mip_formulation} guarantee optimality, the execution time for solving an integer programming problem can be exponential in problem size. Further, the execution time could increase drastically in $m$, as all of the models require at least $O(m^2)$ binary variables and $O(m^2)$ constraints. In order to deal with larger graphs, we propose iterative algorithms for \eqref{def_opt_with_dag} based on Property \ref{property_DAG_topological_order}. Observe that, if we are given a topological order of the nodes, then $Z$ and $O$ are automatically determined in \eqref{def_acyclic_dag_compact}. In other words, we can easily obtain a subset of the arcs such that all of the arcs can be used without creating a cycle. Let $\bar{R}$ be the determined adjacency matrix given topological order $\bar{\pi}$. In detail, we set 
\begin{center}
$\bar{R}_{jk} = 1$ if $\bar{\pi}_j > \bar{\pi}_k$, $\bar{R}_{jk} = 0$ otherwise.
\end{center}
Let $adj(\bar{\pi})$ be the function generating $\bar{R}$ given input topological order $\bar{\pi}$. If we are given $\bar{\pi}$, then we can generate $\bar{R}$ by $adj(\bar{\pi})$, and solving \eqref{def_opt_with_dag} can be written as
\begin{equation}
\label{def_opt_with_dag_fixed}
\displaystyle \min_Y \mbox{ } F(Y) \quad s.t. \quad \bar{R} \geq supp(Y).
\end{equation}
Note that \eqref{def_opt_with_dag_fixed} has acyclic constraint $\bar{R} \geq supp(Y)$, not $\bar{R} = supp(Y)$. The inequality is needed when we try to obtain a sparse solution, i.e., only a subset of the arcs is selected among all possible arcs implied by $\bar{R}$. As long as we satisfy the inequality, $Y$ forms an acyclic subgraph. Hence, $\bar{R}$ can be different from adjacency matrix $supp(Y)$ in an optimal solution of \eqref{def_opt_with_dag_fixed}, and any arc $(j,k)$ such that $\bar{R}_{jk} = 1$ can be selected without creating a cycle. For this reason, we call $\bar{R}$ an \textit{adjacency candidate matrix}.
The algorithms proposed later in this section solve \eqref{def_opt_with_dag_fixed} by providing different $\bar{\pi}$ and $\bar{R} = adj(\bar{\pi})$ in each iteration. In fact, \eqref{def_opt_with_dag_fixed} is separable into $m$ sub problems if $F$ is separable. Let $Y_k$ and $Z_k$ be the $k^{th}$ columns of $Y$ and $R$, respectively, for node $k$. Then, solving
\begin{equation}
\label{def_opt_with_dag_fixed_k}
\displaystyle \min_{Y_k} \mbox{ } F_k(Y_k) \quad s.t. \quad \bar{R}_k \geq supp(Y_k),
\end{equation}
for all $k \in J$ gives the same solution as solving \eqref{def_opt_with_dag_fixed} if $F$ is separable as $F(Y) = \sum_{k \in J} F_k(Y_k)$.

In Section \ref{subsec_swap}, a local improvement algorithm for a given topological order is presented. The algorithm swaps pairs of nodes in the order. In both of the iterative algorithms proposed in Sections \ref{subsec_ir} and \ref{subsec_gd}, we use the local improvement algorithm presented in the following section.

\subsection{Topological Order Swapping Algorithm}
\label{subsec_swap}

Algorithm \ref{algo_swap} tries to improve the solution by swapping the topological order. In each iteration, the algorithm determines the nodes to swap that have order $s_1$ and $s_2$ in Line 3, where $s_2 = s_1 + 1$ implies that we select two nodes which are neighbors in the current topological order. Then in Line 4, the actual node indices $k_1$ and $k_2$ such that $\pi_{k_1} = s_1$ and $\pi_{k_2} = s_2$ are detected. The condition in Line 5 is to avoid meaningless computation when $Y^*$ is sparse. If $|Y^*_{k_2 k_1}| > 0$, we know for sure that $Y^*_{k_2 k_1} = 0$ after swapping the orders of $k_1$ and $k_2$ and thus we will get a different solution. However, if $Y^*_{k_2 k_1} = 0$, we will still have $Y^*_{k_2 k_1} = 0$ after the swap forced by the new order. In Line 6, we create a new topological order $\bar{\pi}$ by swapping nodes $k_1$ and $k_2$ in $\pi^*$. After obtaining adjacency candidate matrix $\bar{R}$ in Line 7, we solve \eqref{def_opt_with_dag_fixed} with $\bar{R}$. It is worth noting that, if $F$ is separable, we only need to solve  \eqref{def_opt_with_dag_fixed_k} with $k = k_1$ and $k_2$ because the values of $\bar{R}$ are the same with $R^*$ except for $k_1$ and $k_2$ as the order difference was 1 in $\pi^*$. In Line 9, we update the best solution if the new solution is better. The iterations continue until there is no improvement in the past $m$ iterations, which implies that we would swap the same nodes if we proceed after this iteration. Algorithm \ref{algo_swap} is illustrated by the following toy example.
\begin{algorithm}[ht]
\caption{TOSA (Topological Order Swapping Algorithm)}        
\label{algo_swap}                           
\begin{algorithmic}[1]   
\vspace{0.1cm}
\REQUIRE $Y'$, $R'$, $\pi'$
\ENSURE Best solution $Y^*, R^*, \pi^*$
\STATE ($Y^*, R^*, \pi^*$) $\gets$ ($Y', R', \pi'$), $t \gets 0$
\STATE \textbf{While} there is an improvement in the past $m$ iterations
\STATE \quad $t \gets t+1$, $s_1 \gets (t \mbox{ mod } (m-1)) + 1$, $s_2 \gets s_1 + 1$
\STATE \quad $(k_1,k_2) \gets$ node indices satisfying $\pi_{k_1} = s_1$ and $\pi_{k_2} = s_2$
\STATE \quad \textbf{If} $|Y^*_{k_2 k_1}|> 0$ \textbf{then}
\STATE \quad \quad $\bar{\pi} \gets$ $\pi^*$, $\bar{\pi}_{k_1}= s_2$, $\bar{\pi}_{k_2} = s_1$
\STATE \quad \quad $\bar{R} \gets$ $adj(\bar{\pi})$
\STATE \quad \quad $\bar{Y} \gets$ solve \eqref{def_opt_with_dag_fixed} with $\bar{R}$
\STATE \quad \quad \textbf{If} $F(\bar{Y}) < F(Y^*)$ \textbf{then} update ($Y^*, R^*, \pi^*$)
\STATE \quad \textbf{End if}
\STATE \textbf{End While}
\end{algorithmic}
\end{algorithm}

\begin{example}
Consider a graph with $m=4$ nodes. Let us assume that inputs are $\pi' = (2,3,1,4)$ with corresponding order $3-1-2-4$,
\begin{center}
$Y' = \left[ \begin{array}{cccc} 0 & 0 & 0.5 & 0 \\ 0 & 0 & 0.5 & 0 \\ 0 & 0 & 0 & 0 \\ 0.4 & 0.8 & 0.1 & 0 \end{array} \right] $, and $R' = \left[ \begin{array}{cccc} 0 & 0 & 1 & 0 \\ 0 & 0 & 1 & 0 \\ 0 & 0 & 0 & 0 \\ 1 & 1 & 1 & 0 \end{array} \right]. $
\end{center}
In iteration 1, $t=1$, $s_1 = 1$, $s_2 = 2$, $k_1 = 3$, and $k_2 = 1$. Hence, we are swapping nodes 3 and 1. Since $|Y^*_{13}| = 0.5 > 0$, $\bar{\pi} = (1,3,2,4)$ is created in Line 6, where the associated order is $1-3-2-4$. If $\bar{\pi}$ gives an improved objective function value, then $\pi^*$ is updated in Line 9. Let us assume that $\pi^*$ is not updated. In iteration 2, $t=2$, $s_1 = 2$, $s_2 = 3$, $k_1 = 1$, and $k_2 = 2$. Since $|Y^*_{21}| = 0$, Lines 6 - 9 are not executed. $\hfill$ $\square$
\end{example}

\subsection{Iterative Reorering Algorithm}
\label{subsec_ir}

We propose an iterative reordering algorithm based on Property \ref{property_DAG_topological_order}, which solves \eqref{def_opt_with_dag_fixed} in each iteration aiming to optimize \eqref{def_opt_with_dag}. In each iteration of the algorithm, all nodes are sorted based on scores defined by (\rmnum{1}) merit scores of the arcs, (\rmnum{2}) historical choice of the arcs (used as weights), (\rmnum{3}) and some random components. Then the sorted node order is directly used as a topological order. The selected arcs by the topological order give updates on arc weights. Let us first define notation.
\begin{enumerate}[noitemsep]
\item[] $\nu = $ uniform random variable on $[\nu_{lb}, \nu_{ub}]$, $\nu_{lb} < 1 < \nu_{ub}$
\item[] $\rho_{jk}$ = pre-determined merit score of arc $(j,k)$ for $j \in J,k \in J$
\item[] $w_{jk}$ = weight of arc $(j,k)$ for $j \in J,k \in J$
\item[] $c_k$ = score of node $k$, $k \in J$ 
\end{enumerate}
The range $[\nu_{lb}, \nu_{ub}]$ of the uniform random variable $\nu$ balances the randomness and structured scores. Note that $\rho_{jk}$ should be determined based on the data and the characteristic of the problem considered, where larger $\rho_{jk}$ implies that arc $(j,k)$ is attractive. Based on the arc merit scores $\rho$, the score for node $k$, $k \in J$, is defined as
\begin{equation}
\label{def_score}
c_k = \nu \cdot \big( \sum_{j \in J^k} w_{jk} \rho_{jk} \big), \quad k \in J,
\end{equation}
which can be interpreted as a weighted summation of $\rho_{jk}$'s multiplied by perturbation random number $\nu$. Hence, nodes with high scores are attractive. Initially, all arcs have equal weights and the weights are updated in each iteration based on the topological order in the iteration. If $\bar{R} = adj(\bar{\pi})$ is the adjacency candidate matrix in the iteration, then, the weights are updated by
\begin{equation}
\label{def_weight_update_ver1}
w_{jk} = w_{jk} + 1, \quad \mbox{ if } \bar{R}_{jk} = 1.
\end{equation}

The overall algorithmic framework is summarized in Algorithm \ref{algo_IR}. In Line 1, weights $w_{jk}$'s are initialized to 1 and $\bar{t}$, which counts the number of iterations without a best solution update, is initialized. Also, a random order $\pi^*$ of the nodes is generated, and the corresponding solution becomes the best solution. In each iteration, first node scores $c_k$'s are calculated (Line 3), then topological order $\bar{\pi}$ is obtained by sorting the nodes, and finally adjacency candidate matrix $\bar{R}$ is generated (Line 4). Then, in Lines 5 and 6, solution $\bar{Y}$ is obtained by solving \eqref{def_opt_with_dag_fixed} with $\bar{R}$ and the best solution is updated if available. In Lines 7 - 10, \textit{TOSA} is executed if the current solution is within a certain percentage $\alpha$ from the best solution. Lines 11 and 12 update $\bar{t}$, and Line 13 updates $w_{jk}$'s. This ends the iteration and the algorithm continues until $\bar{\pi}$ is converged or there is no update of the best solution in the last $t^*$ iterations. Algorithm \ref{algo_IR} is illustrated by the following toy example.

\begin{algorithm}[ht]
\caption{IR (Iterative Reordering)}        
\label{algo_IR}                           
\begin{algorithmic}[1]   
\vspace{0.1cm}
\REQUIRE Merit score $\rho \in \mathbb{R}^{m \times m}$, termination parameter $t^*$, \textit{TOSA} execution parameter $\alpha$
\ENSURE Best solution $Y^*, R^*, \pi^*$
\STATE $w_{jk} \gets 1$, $\pi^* \gets$ a random order, $R^* \gets adj(\pi^*)$, $\bar{\pi} \gets \pi^*$, $Y^* \gets$ solve \eqref{def_opt_with_dag_fixed} with $R^*$, $\bar{t} \gets 0$
\STATE \textbf{While} (\rmnum{1}) $\bar{\pi}$ is not convergent or (\rmnum{2}) $\bar{t} < t^*$
\STATE \quad Calculate score $c_k$ by \eqref{def_score}
\STATE \quad $\bar{\pi} \gets$ sort nodes with respect to $c_k$, $\bar{R} \gets adj(\bar{\pi})$
\STATE \quad $\bar{Y} \gets$ solve \eqref{def_opt_with_dag_fixed} with $\bar{R}$ \vspace{0.1cm}
\STATE \quad \textbf{If} $F(\bar{Y}) < F(Y^*)$ \textbf{then} update $(Y^*,R^*,\pi^*)$  \vspace{0.1cm}
\STATE \quad \textbf{If} $F(\bar{Y}) < F(Y^*)\cdot(1 + \alpha)$
\STATE \quad \quad $(Y',R',\pi') \gets \textit{TOSA}(\bar{Y},\bar{R},\bar{\pi})$,
\STATE \quad \quad \textbf{If} $F(Y') < F(Y^*)$ \textbf{then} update $(Y^*,R^*,\pi^*)$
\STATE \quad \textbf{End If} \vspace{0.1cm}
\STATE \quad \textbf{If} $(Y^*,R^*,\pi^*)$ is updated \textbf{then} $\bar{t} \gets 0$
\STATE \quad \textbf{Else} $\bar{t} \gets \bar{t} + 1$ \vspace{0.1cm}
\STATE \quad Update weights by \eqref{def_weight_update_ver1}
\STATE \textbf{End While}
\end{algorithmic}
\end{algorithm}

\begin{example}
Consider a graph with $m=3$ nodes. In the current iteration, let us assume that we are given 
\begin{center}
$\rho = \left[ \begin{array}{ccc} 0 & 0.5 & 0.5 \\ 0.2 & 0 & 0.2 \\ 0.3 & 0.3 & 0 \end{array} \right] $ and $w = \left[ \begin{array}{ccc} 0 & 1 & 2 \\ 1 & 0 & 1 \\ 2 & 1 & 0 \end{array} \right] $.
\end{center}
Note that we have $\sum_{j \in J^1} w_{j1}\rho_{j1} = 0.2 \cdot 1 + 0.3 \cdot 2 = 0.8$, $\sum_{j \in J^2} w_{j2}\rho_{j2} = 0.5 \cdot 1 + 0.3 \cdot 1 = 0.8$, and $\sum_{j \in J^3} w_{j3}\rho_{j3} = 0.4 \cdot 2  + 0.2 \cdot 1 = 1$. If random numbers ($\nu$) are 0.9, 1.1, 0.8 for nodes 1,2, and 3, respectively, then by \eqref{def_score}, $c_1 = 0.9 \cdot 0.8 = 0.72$, $c_2 = 1.1 \cdot 0.8 = 0.88$, and $c_3 = 0.8 \cdot 1 = 0.8$. Then in Line 4, we obtain $\bar{\pi} = (3,1,2)$, with corresponding order $2-3-1$, and $\bar{R} = [ 0,1,1; 0,0,0; 0,1,0]$. After obtaining $\bar{Y}$ and updating the best solution in Lines 5-12, the weights are updated by \eqref{def_weight_update_ver1} as follows.
\begin{center}
$w_{new}= \left[ \begin{array}{ccc} 0 & 1 & 2 \\ 1 & 0 & 1 \\ 2 & 1 & 0 \end{array} \right] + \bar{R} = \left[ \begin{array}{ccc} 0 & 2 & 3 \\ 1 & 0 & 1 \\ 2 & 2 & 0 \end{array} \right] $ 
\end{center}
This ends the current iteration. $\hfill$ $\square$
\end{example}

\subsection{Gradient Descent Algorithm}
\label{subsec_gd}

In this section, we propose a gradient descent algorithm based on Property \ref{property_DAG_topological_order}. The algorithm iteratively executes: (\rmnum{1}) moving toward an improving direction by gradients, (\rmnum{2}) DAG structure is recovered and topological order is obtained by a projection step. The algorithm is based on the standard gradient descent framework while the projection step takes care of the acyclicity constraints by generating a topological order from the current (possibly cyclic) solution matrix. In order to distinguish the solutions with and without the acyclicity property, we use the following notation.
\begin{enumerate}[noitemsep]
\item[] $U^t \in \mathbb{R}^{m \times m}$ = decision variable matrix without acyclicity requirement in iteration $t$
\item[] $Y^t \in \mathbb{R}^{m \times m}$ = decision variable matrix satisfying  $G(supp(Y^t)) \in \mathcal{A}$ in iteration $t$
\end{enumerate}
Let $\gamma^t$ be the step size in iteration $t$, $\nabla F(Y^t)$ be the derivative of $F$ at $Y^t$, and $G^t \in \mathbb{R}^{m \times m}$ be a weight matrix that weighs each element. We assume $\|\nabla F(Y^t)\|_{\infty} \leq M_1$ for a constant $M_1$, where $\| \cdot \|_{\infty}$ is the uniform (infinity) norm. The update formula
\begin{equation}
\label{def_update_y}
U^{t} = Y^t - \gamma^t \big[ \nabla F(Y^t) \circ G^t \big], \quad t \geq 0,
\end{equation}
updates $Y^t$ based on the weighted gradient, where $\circ$ represents the entrywise or Hadamard product of the two matrices. Given topological order $\pi^t$, we define $G^t$ as
\begin{equation}
\label{def_grad_weight}
\displaystyle G_{jk}^t = \Big( 1 + \frac{1}{\pi_k^t} \Big)^{\pi_k^t}, \quad j \in J^k
\end{equation}
to balance gradients of the nodes with different orders (small and large values of $\pi_k^t$). For nodes $k_1$ and $k_2$ with $\pi_{k_1}^t = 1$ and $\pi_{k_2}^t = m$, most of the gradients for node $k_1$ are zero and most of the gradients for node $k_2$ are nonzero. Weight \eqref{def_grad_weight} tries to adjust this gap. Note that we have $2 \leq G_{jk}^t \leq e$ for any large $m$. Since $U^t$ may not satisfy acyclic constraints, in order to obtain a DAG, the algorithm needs to solve the projection problem
\begin{equation}
\label{def_opt_unconstrained_acyclic}
\displaystyle Y^* = \mbox{argmin}_Y \mbox{ } \| Y - U^t\|_2^2 \quad s.t. \quad G(supp(Y)) \in \mathcal{A},
\end{equation}
where $\| \cdot \|_2$ is the $L_2$ norm.

\begin{proposition} 
If $U^t$ is arbitrary, then optimization problem \eqref{def_opt_unconstrained_acyclic} is $\mathcal{NP}$-hard.
\end{proposition}
\begin{proof}
Recall that feedback arc set is $\mathcal{NP}$-complete \cite{karp1972reducibility} and maximum acyclic subgraph is the dual of the feedback arc set problem. With $U^t = 1$, \eqref{def_opt_unconstrained_acyclic} becomes the weighted maximum acyclic subgraph problem. Therefore, \eqref{def_opt_unconstrained_acyclic} is $\mathcal{NP}$-complete.
\end{proof}

Because solving \eqref{def_opt_unconstrained_acyclic} to optimality does not guarantee an optimal solution for \eqref{def_opt_with_dag}, we use a greedy strategy to solve \eqref{def_opt_unconstrained_acyclic}. The greedy algorithm, presented in Algorithm \ref{algo_greedy}, sequentially determines and fixes the topological order of a node where in each iteration the problem is solved optimally given the currently fixed nodes and corresponding orders. The detailed derivations of the algorithm and the proof that each iteration is optimal, given already fixed node orders, are available in Appendix \ref{appendix_greedy_algo}. In other words, we show that Line 3 is `locally' optimal, i.e., it selects the best next node given that the order $q + 1, q+2, \cdots, m$ is fixed. In each iteration, in Line 3, the algorithm first calculates score $\sum_{j \in \bar{J}} (\bar{U}_{jk}^t)^2$ for each node $k$ in $\bar{J}$ and picks node $k^*$ with the minimum value. Then, in Line 4, the order of the selected node is fixed to $q$. The fixed node is then excluded from the active set $\bar{J}$ and iterate $q$ is decreased by 1 in Line 5. At the end of the algorithm, we can determine $\bar{Y}$ based on the order $\bar{\pi}$ determined and \eqref{def_opt_y_structure} in Appendix \ref{appendix_greedy_algo}. We illustrate Algorithm \ref{algo_greedy} by the following example.

\begin{algorithm}[ht]
\caption{Greedy}        
\label{algo_greedy}                           
\begin{algorithmic}[1]   
\vspace{0.1cm}
\REQUIRE $U^t \in \mathbb{R}^{m \times m}$
\ENSURE $\bar{Y}$ feasible to \eqref{def_opt_unconstrained_acyclic}, topological order $\bar{\pi}$
\STATE $q \gets m$, $\bar{J} \gets J$
\STATE \textbf{While} $\bar{J} \neq \emptyset$
\STATE \quad $\displaystyle k^* = \mbox{argmin}_{k \in \bar{J}} \Big\{ \sum_{j \in \bar{J}} (U_{jk}^t)^2 \Big\}$
\STATE \quad $\bar{\pi}_{k^*} = q$
\STATE \quad $\bar{J} \gets \bar{J} \setminus \{k^*\}, q \gets q-1$
\STATE \textbf{End While}
\STATE Determine $\bar{Y}$ by \eqref{def_opt_y_structure} in Appendix \ref{appendix_greedy_algo}
\end{algorithmic}
\end{algorithm}
\begin{example}
Consider a graph with $m=3$ nodes. Given $U^t$, the algorithm returns $\bar{Y}$ presented in the following.
\begin{center}
$U^t = \left[ \begin{array}{ccc} 0 & 1 & 2 \\ 4 & 0 & 2 \\ 5 & 2 & 0 \end{array} \right] $ \quad $\bar{Y} = \left[ \begin{array}{ccc} 0 & 0 & 0 \\ 4 & 0 & 2 \\ 5 & 0 & 0 \end{array} \right] $
\end{center}
Algorithm \ref{algo_greedy} starts with $q = 3$ and $\bar{J} = \{1,2,3\}$. In iteration 1, node 2 is selected to have $\pi_2 = 3$ based on $\mbox{argmin} \{ 4^2 + 5^2,1^2 + 2^2, 2^2 + 2^2\}$. Then, set $\bar{J}$ and integer $q$ are updated to $\bar{J} = \{1,3\}$ and $q=2$. In iteration 2, node 3 is selected to have $\pi_3 = 2$ based on $\mbox{argmin} \{ 5^2,2^2\}$. Then, set $\bar{J}$ and integer $q$ are updated to $\bar{J} = \{1\}$ and $q=1$. In iteration 3, node 1 is selected. Hence, we have node order 1-3-2 and we obtain $\bar{Y}$ presented above with objective function value $\| \bar{Y} - U^t \|_2^2 = 1^2 + 2^2 + 2^2 = 9$. $\hfill$ $\square$
\end{example}

The overall gradient descent algorithm for \eqref{def_opt_with_dag} is presented in Algorithm \ref{algo_grad_desc}. In Line 1, the algorithm generates a random order $\pi^*$ and obtain corresponding $R^*$ and $Y^*$ and save them as the best solution. In each iteration of the loop, Lines 3-6 follow the standard gradient descent algorithm. The weighted gradient $H^t$ is calculated in Line 3, and the step size is determined in Line 4 based on the ratio between $\max_{j \in J^k, k \in J } |H_{jk}^t|$ and $\max_{j \in J^k, k \in J } |Y_{jk}^t|$. In Line 5, the solution is updated based on the weighted gradient and, in Line 6, the greedy algorithm is used to obtain the projected solution and the topological order. Observe that we do not directly use the projected solution. This is because the projected solution is not necessarily optimal given $\pi^{t+1}$. Hence, in Line 7, a new solution $Y^{t+1}$ is obtained based on $\pi^{t+1}$. In Lines 9 - 12, $TOSA$ is executed if the current solution is within a certain percentage from the current best solution. Lines 13 and 14 update $\bar{t}$ and Line 15 copies $Y^*$ to $Y^{t+1}$ if $\bar{t} \geq t_2^*$ in order to focus on the solution space near $Y^*$. The algorithm continues until $Y^t$ is convergent or $\bar{t} \geq t_1^*$.

\begin{algorithm}[ht]
\caption{GD (Gradient Descent)}        
\label{algo_grad_desc}                           
\begin{algorithmic}[1]   
\vspace{0.1cm}
\REQUIRE Parameters $t_1^*$ and $t_2^*$, \textit{TOSA} execution parameter $\alpha$
\ENSURE Best solution $Y^*, R^*, \pi^*$
\STATE $t \gets 1$, $\bar{t} \gets 0$, $\pi^* \gets$ a random order, $R^* \gets adj(\pi^*)$, $Y^* \gets$ solve \eqref{def_opt_with_dag_fixed} with $R^*$
\STATE \textbf{While} (\rmnum{1}) $Y^t$ is not convergent or (\rmnum{2}) $\bar{t} < t_1^*$
\STATE \quad $H^t \gets \nabla F(Y^t) \circ G^t$, $G^t$ defined in \eqref{def_grad_weight}
\STATE \quad $\gamma^t \gets \frac{\|H^t\|_{\infty}}{\|Y^t\|_{\infty}} \big/ \sqrt{t}$
\STATE \quad $U^{t} \gets Y^t - \gamma^t H^t$ 
\STATE \quad $\pi^{t+1} \gets$ \textit{Greedy($U^t$)} \label{algo_grad_desc_greedy}
\STATE \quad $Y^{t+1} \gets$ solve \eqref{def_opt_with_dag_fixed} with $R^{t+1} = adj(\pi^{t+1})$
\STATE \quad \textbf{If} $F(Y^{t+1}) < F(Y^*)$ \textbf{then} $(Y^*,R^*,\pi^*) \gets (Y^{t+1},R^{t+1},\pi^{t+1})$  \vspace{0.1cm}
\STATE \quad \textbf{If} $F(\bar{Y}) < F(Y^*)\cdot(1 + \alpha)$
\STATE \quad \quad $(Y',R',\pi') \gets TOSA(Y^{t+1},R^{t+1},\pi^t)$,
\STATE \quad \quad \textbf{If} $F(Y') < F(Y^*)$ \textbf{then} $(Y^*,R^*,\pi^*) \gets (Y',R',\pi'), (Y^{t+1},R^{t+1},\pi^{t+1}) \gets (Y',R',\pi')$ 
\STATE \quad \textbf{End If} \vspace{0.1cm}
\STATE \quad \textbf{If} $(Y^*,R^*,\pi^*)$ is updated \textbf{then} $\bar{t} \gets 0$
\STATE \quad \textbf{Else} $\bar{t} \gets \bar{t} + 1$ \vspace{0.1cm}
\STATE \quad \textbf{If} $\bar{t} \geq t_2^*$ \textbf{then} $Y^{t+1} \gets Y^*$ \vspace{0.1cm}
\STATE \quad $t \gets t+1$
\STATE \textbf{End While}
\end{algorithmic}
\end{algorithm}

In gradient based algorithms, it is common to have $\gamma^t$ depend only on t, but in our case dependency on $H^t$ and $Y^t$ is justifiable since we multiply the gradient by $G^t$. We next show the convergence of $Y^t$ in Algorithm \ref{algo_grad_desc} when $t_1^* = t_2^* = \infty$. This makes the algorithm not to terminate unless $Y^t$ has converged and modification of $Y^t$ in Line 15 is not executed. Further, we assume the following for the analysis.
\begin{assumption}
\label{assumption_eps}
For any non-zero element $Y^t_{jk}$, $j,k \in J$, of $Y^t$ in any iteration $t$, we assume $\varepsilon < |Y_{jk}^t| < M_2 $, where $\varepsilon$ is a small positive number and $M_2$ is a large enough number.
\end{assumption}
Note that Assumption \ref{assumption_eps} is a mild assumption, as ignoring near-zero values of $Y^t$ happens in practice anyway due to finite precision. For notational convenience, let $L^t = \gamma^t \nabla F(Y^t) \circ G^t = \gamma^t H^t$ be the second term in \eqref{def_update_y}. Then, $U^t$ can be written as $U^t = Y^t - L^t = Y^t - \gamma^t H^t$. In the following lemma, we show that the node orders converge.

\begin{lemma}
\label{lemma_nochange_order}
If $t$ is sufficiently large satisfying $\sqrt{t} > \frac{(M_1 e)^2}{\varepsilon (\sqrt{M_2^2 + \varepsilon^2/m} - M_2)}  $, then $\pi^t = \pi^{t+1}$. 
\end{lemma}
\begin{proof}
Recall that $Y^t$ is obtained by solving \eqref{def_greedy_prob_algorithm} and we know the corresponding node order $\pi^t$. Let $k_1,k_2,\cdots,k_m$ be the node indices defined based on $\pi^t$. In other words, node $k_1$ appears first, followed by nodes $k_2$, $k_3$, and so on in the topological order $\pi^t$. In the proof, we show that there is no change in the node order when the condition $\sqrt{t} > \frac{(M_1 e)^2}{\varepsilon (\sqrt{M_2^2 + \varepsilon^2/m} - M_2)}  $ is met, where $M_1$ and $M_2$ are the upper bounds for $\| \nabla F(Y^t)\|_{\infty}$ and $\| Y^t \|_{\infty}$, respectively, as assumed. We first derive
\begin{equation}
\label{eqn_lemma_nochange_order}
\|L^t\|_{\infty} = \Big\| \gamma^t H^t \Big\|_{\infty} = \Big\| \frac{1}{\sqrt{t}} \frac{\|H^t\|_{\infty}}{ \|Y^t\|_{\infty}} H^t \Big\|_{\infty} \leq \frac{1}{\sqrt{t}} \frac{\big(\|H^t\|_{\infty}\big)^2}{ \|Y^t\|_{\infty}} < \frac{1}{\sqrt{t}} \frac{(M_1 e)^2}{\varepsilon},
\end{equation}
where the last inequality holds since (\rmnum{1}) $ \|Y^t\|_{\infty} > \varepsilon$ by Assumption \ref{assumption_eps}, (\rmnum{2}) $\|H^t\|_{\infty} = \| \nabla F(Y^t) \circ G^t \|_{\infty} \leq M_1 e$ because $\|\nabla F(Y^t)\|_{\infty} \leq M_1$ by the assumption and $\| G^t \|_{\infty} \leq e$, where $e$ is natural number.

Now let us consider $q = \bar{q}$ in Algorithm \ref{algo_greedy} to decide node order $\bar{q}$ in iteration $t+1$ and assume $\pi_{k_r}^{t} = \pi_{k_r}^{t+1}$ for $r = m,m-1,\cdots,\bar{q}-1$. Note that we currently have $\bar{J} = \{ k_1, k_2, \cdots, k_{\bar{q}}\}.$

\begin{enumerate}
\item For $k_{\bar{q}}$, we derive $\sum_{j \in \bar{J}} (U_{jk_{\bar{q}}}^t)^2 = \sum_{j \in \bar{J}} (Y_{jk_{\bar{q}}}^t - L_{jk_{\bar{q}}}^t)^2 = \sum_{j \in \bar{J}} (L_{jk_{\bar{q}}}^t)^2 < m \big[\frac{(M_1 e)^2}{\sqrt{t} \varepsilon} \big]^2,$ where the second equality holds since $Y_{jk_{\bar{q}}}^t = 0$ for all $j \in \bar{J}$ since $\pi_{k_{\bar{q}}} = \bar{q}$ and no arc can be used to the nodes in $\bar{J}$, and the last inequality holds due to \eqref{eqn_lemma_nochange_order} and $| \bar{J}| \leq m$.
\item For all other nodes $k_r \in \bar{J} \setminus \{k_{\bar{q}}\}$, we derive

\begin{tabular}{lll}
$\sum_{j \in \bar{J}} (U_{jk_r}^t)^2$ & $=$ & $\sum_{j \in \bar{J}} (Y_{jk_r}^t - L_{jk_r}^t)^2$\\
		& $=$ & $\sum_{j \in \bar{J}} (Y_{jk_r}^t)^2 + \sum_{j \in \bar{J}} (L_{jk_r}^t)^2 - 2 \sum_{j \in \bar{J}} Y_{jk_r}^t \cdot L_{jk_r}^t$\\
		& $>$ & $\sum_{j \in \bar{J}} (Y_{jk_r}^t)^2 - 2 \sum_{j \in \bar{J}} Y_{jk_r}^t \cdot L_{jk_r}^t$\\
		& $>$ & $\varepsilon^2 - 2 \sum_{j \in \bar{J}} Y_{jk_r}^t \cdot L_{jk_r}^t$\\
		& $\geq$ & $\varepsilon^2 - 2 \sum_{j \in \bar{J}} |Y_{jk_m}^t \cdot L_{jk_r}^t|$\\
		& $>$ & $\varepsilon^2 - 2 M_2  \frac{m(M_1 e)^2}{\sqrt{t} \varepsilon}$\\
\end{tabular}

where the fourth line holds due to $|Y_{jk_r}^t| > \varepsilon$ by Assumption \ref{assumption_eps}, and the sixth line holds due to $|Y_{jk}^t| \leq M_2$, $|\bar{J}| \leq m$, and $|L_{jk_r}^t| < \frac{(M_1 e)^2}{\sqrt{t} \varepsilon}$ by \eqref{eqn_lemma_nochange_order}.
\end{enumerate}
Combining the two results for $k_{\bar{q}}$ and $k_r \in \bar{J} \setminus \{k_{\bar{q}} \}$, we obtain $\sum_{j \in \bar{J}} (U_{jk_{\bar{q}}}^t)^2 <  m \big[\frac{(M_1 e)^2}{\sqrt{t} \varepsilon} \big]^2 < \varepsilon^2 - 2 M_2  \frac{m(M_1 e)^2}{\sqrt{t} \varepsilon} < \sum_{j \in \bar{J}} (U_{jk_r}^t)^2$, for any $r \in \{1,2,\cdots,\bar{q}-1\}$, where the second inequality holds due to the condition $\sqrt{t} > \frac{(M_1 e)^2}{\varepsilon (\sqrt{M_2^2 + \varepsilon^2/m} - M_2)}$. The result implies that we must have $\pi_{k_{\bar{q}}}^{t+1} = \pi_{k_{\bar{q}}}^{t} = \bar{q}$ by Line 3 in Algorithm \ref{algo_greedy}. 

Note that, when $\bar{q} = m$, we have $J = \bar{J}$ and the assumption of $\pi_{k_r}^{t} = \pi_{k_r}^{t+1}$ for $r = m,m-1,\cdots,\bar{q}-1$ automatically holds. By iteratively applying the above derivation technique from $q = m$ to 1, we can show that $\pi^t = \pi^{t+1}$.
\end{proof}

When \eqref{def_opt_with_dag_fixed} is solved with the identical node orders, the resulting solutions are equivalent. Hence, the following proposition holds.
\begin{proposition} 
\label{proposition_converged_y_t}
In Algorithm \ref{algo_grad_desc}, $Y^t$ converges in $t$.
\end{proposition}

\section{Estimation of Gaussian Bayesian Networks}
\label{section_reg_network}

In this section, we introduce the Gaussian Bayesian network learning problem, which follows the form of \eqref{def_opt_with_dag}. The goal is to learn or estimate unknown structure between the nodes of a graph, where the error is normally distributed. The network can be estimated by optimizing a score function, testing conditional independence, or a mix of the two, as described in Section \ref{section_intro}. Among the three categories, we select the score based approach with the $L_1$-penalized least square function recently studied in Han \textit{et al.} \cite{han2016estimation}.

Let $X \in \mathbb{R}^{n \times m}$ be a data set with $n$ observations and $m$ features. Let $I = \{1,\cdots,n\}$ and $J = \{1,\cdots,m\}$ be the index set of observations and features, respectively. For each $k \in J$, we build a regression model in order to explain feature $k$ using a subset of variables in $J^k$. In other words, we set feature $k$ as the response variable and a sparse subset of $J^k$ as explanatory variables of the regression model for variable $k$. In order to obtain a subset of $J^k$, the LASSO penalty function is added. Considering regression models for all $k \in J$ together, the problem can be represented on a graph. Each feature is a node in the graph, and the directed arc from node $j$ to node $k$ represents explanatory and response variable relationship between node $j$ and $k$. The goal is to minimize the sum of penalized SSE over all regression models for $k \in J$, while the selected arcs do not create a cycle.

Let $\beta_{jk}$, $j \in J^k$, $k \in J$, be the coefficient of attribute $j$ for dependent variable $k$. Then the problem can be written as
\begin{equation}
\displaystyle \min_{\beta} \frac{1}{n} \sum_{k \in J} \sum_{i \in I} (x_{ik} - \sum_{j \in J^k} \beta_{jk} x_{ij})^2 + \lambda \sum_{k \in J} \sum_{j \in J^k} |\beta_{jk}| \quad s.t. \quad G(supp(\beta)) \in \mathcal{A}, \label{formualtion_regression_network}
\end{equation}
which follows the form of \eqref{def_opt_with_dag}. In Han \textit{et al.} \cite{han2016estimation}, individual weights are used for the penalty term, i.e., $\lambda \sum_{k \in J} \sum_{j \in J^k} w_{jk} |\beta_{jk}|$, however, in the computational experiment, we set all weights equal to 1 for simplicity.

Let $Z_{jk} = 1$ if attribute $j$ is used for dependent variable $k$ and $Z_{jk} = 0$ otherwise. Then, we can formulate \textsf{MIPto} for \eqref{formualtion_regression_network} as
\begin{subequations}
\label{MIP_regression_network}
\begin{align}
\min \quad & \frac{1}{n} \sum_{k \in J} \sum_{i \in I} (x_{ik} - \sum_{j \in J^k} \beta_{jk} x_{ij})^2 +\lambda \sum_{k \in J} \sum_{j \in J^k} |\beta_{jk}| \label{MIP_regression_network_a}\\
s.t.\quad & |\beta_{jk}| \leq M Z_{jk}, & & j \in J^k, k \in J, \label{MIP_regression_network_supp}\\
& Z_{jk} - m Z_{kj} \leq \sum_{r \in J} r ( O_{kr} - O_{jr}), && j \in J^k, k \in J, \\
& Z_{jk} + Z_{kj} \leq 1, && j \in J^k, k \in J, \\
& \sum_{q \in J} O_{kq} = 1, & & k \in J, \\
& \sum_{k \in J} O_{kq} = 1, & & q \in J, \\
& \textstyle Z,O \in \{0,1\}^{m \times m},\\
& \beta_{jk} \mbox{ not restricted}, & & j \in J^k \cup \{0\}, k \in J, \label{MIP_regression_network_h}
\end{align}
\end{subequations}
where $M$ is a large constant. Note that \eqref{MIP_regression_network_supp} is the linear constraint corresponding to $Z = supp(Y)$ in \eqref{def_acyclic_dag_compact}. Similarly, \eqref{MIP_regression_network_a}, \eqref{MIP_regression_network_supp}, \eqref{mip_acyclic_baharev_triangle_c} - \eqref{mip_acyclic_baharev_triangle_e} and \eqref{MIP_regression_network_h} can be used to formulate \textsf{MIPin} for \eqref{formualtion_regression_network}. For \textsf{MIPcp}, \eqref{MIP_regression_network_a}, \eqref{MIP_regression_network_supp}, \eqref{MIP_regression_network_h}, and the constraints in \eqref{def_opt_with_dag_cuttingplane} can be used for \eqref{formualtion_regression_network}.

Note that $M$ in \eqref{MIP_regression_network_supp} plays an important role in computational efficiency and optimality. If $M$ is too small, the MIP model cannot guarantee optimality. If $M$ is too large, the solution time can be as large as enumeration. The algorithm for getting a valid value for $M$ in Park and Klabjan \cite{ParkKlabjanMIPreg} can be used. However, the valid value of big $M$ for multiple linear regression is often too large \cite{ParkKlabjanMIPreg}. For \eqref{MIP_regression_network}, we observed that a simple heuristic presented in Section \ref{sec_experiment} works well.

In each iteration of \textsf{IR} (Algorithm \ref{algo_IR}) and \textsf{GD} (Algorithm \ref{algo_grad_desc}), we are given topological order $\bar{\pi}$ and matrix $\bar{R} = adj(\bar{\pi})$. Let $S^k = \{j \in J^k | \bar{R}_{jk} = 1 \}$ be the set of selected candidate arcs for dependent variable $k$. Given fixed $\bar{R}$, \eqref{formualtion_regression_network} is separable into $m$ LASSO linear regression problems

\begin{equation}
\label{formualtion_regression_network_k_fixZ}
\min \quad \frac{1}{n} \sum_{i \in I} (x_{ik} - \sum_{j \in S^k} \beta_{jk} x_{ij})^2 + \lambda \sum_{j \in S^k} |\beta_{jk}|, \quad k \in J.
\end{equation}

\section{Computational Experiment}
\label{sec_experiment}

For all computational experiments, a server with two Xeon 2.70GHz CPUs and 24GB RAM is used.  Although there are many papers studying Bayesian network learning with various error measures and penalties, here we focus on minimizing the LASSO type objective (SSE and penalty) and we picked one of the latest paper of Han \textit{et al.} \cite{han2016estimation} with the same objective function as benchmark.

The MIP models \textsf{MIPcp}, \textsf{MIPin}, and \textsf{MIPto} are implemented with CPLEX 12.6 in C\#. For \textsf{MIPcp}, instead of implementing the original cutting plane algorithm, we use CPLEX Lazy Callback, which is similar to the cutting plane algorithm. Instead of solving \eqref{def_opt_with_dag_cuttingplane_iter} to optimally from scratch in each iteration, we solve \eqref{def_opt_with_dag_cuttingplane} with Lazy Callback, which allows updating (adding) constraints (cycle prevention constraints) in the process of the branch and bound algorithm whenever an integer solution with cycles is found. Given a solution with the cycles, we detect all cycles and add cycle prevention constraints for the detected cycles. 

For \textsf{MIPcp}, \textsf{MIPin}, and \textsf{MIPto}, we set big $M$ as follows. Given $\lambda$, we solve \eqref{formualtion_regression_network} without acyclic constraints. Hence, we are allowed to use all arcs in $J^k$ for each model $k \in J$. Then, we obtain the estimated upper bound for big $M$ by
\begin{equation}
\label{eqn_bigM}
M = 2 \Big( \max_{j \in J^k, k \in J} |\beta_{jk}| \Big).
\end{equation}
We observed that the above formula gives large enough big $M$ for all cases in the following experiment. In Appendix \ref{Appendix_max_coef}, we present comparison of regression coefficients of implanted network (DAG) with big M values by \eqref{eqn_bigM}. The result shows that the big M value in \eqref{eqn_bigM} is always valid for all cases considered.

We compare our algorithms and models with the algorithm in Han \textit{et al.} \cite{han2016estimation}, which we denote as \textsf{DIST} here. Their algorithm starts with neighborhood selection (NS), which filters unattractive arcs and removes them from consideration. The procedure is specifically developed for high dimensional variable selection when $m$ is much larger than $n$. In our experiment, many instances considered are not high dimensional and some have dense solutions. Further, by filtering arcs, there exists a probability that an arc in the optimal solution can be removed. Hence, we deactivated the neighborhood selection step of their original algorithm, where the R script of the original algorithm is available on the journal website.

For \textsf{GD} and \textsf{IR}, the algorithms are written in R \cite{Rstat}. We use \textsf{glmnet} package \cite{glmnet} function \textsf{glmnet} for solving LASSO linear regression problems in \eqref{formualtion_regression_network_k_fixZ}. For \textsf{IR}, we use parameters $\alpha = 0.01$, $t^* = 10$, $\nu_{lb} = 0.8$, and $\nu_{ub} = 1.2$. For \textsf{GD}, we use parameters $\alpha = 0.01$, $t_1^* = 10$, and $t_2^* = 5$. Because both \textsf{GD} and \textsf{IR} start with a random solution, they perform different with different random solutions. Further, since we observe that the execution time of \textsf{GD} and \textsf{IR} are much faster than \textsf{DIST}, we decided to run \textsf{GD} and \textsf{IR} with 10 different random seeds and report the best solution. To emphasize the number of different random seeds for \textsf{GD} and \textsf{IR}, we use the notation \textsf{GD10} and \textsf{IR10} in the rest of the section.

We first test all algorithms with synthetic instances generated using R package \textsf{pcalg} \cite{pcalgpackage}. Function \textsf{randomDAG} is used to generate a DAG and function \textsf{rmvDAG} is used to generate multivariate data with the standard normal error distribution. First, a DAG is generated by \textsf{randomDAG} function. Next, the generated DAG and random coefficients are used to create each column (with standard normal error added) by \textsf{rmvDAG} function which uses linear regression as the underlying model. After obtaining the data matrix from the package, we standardize each column to have zero mean with standard deviation equal to one. The DAG used to generate the multivariate data is considered as the true structure or true arc set while it may not be the optimal solution for the score function. The random instances are generated for various parameters described in the following.
\begin{enumerate}[noitemsep]
\item[] $m$: number of features (nodes)
\item[] $n$: number of observations
\item[] $s$: expected number of true arcs per node
\item[] $d$: expected density of the adjacency matrix of the true arcs
\end{enumerate}
By changing the ranges of the above parameters, three classes of random instances are generated.
\begin{enumerate}[noitemsep]
\item[] \textsf{Sparse data sets}: The expected total number of true arcs is controlled by $s$ and most of the instances have a sparse true arc set. We use $n \in \{ 100,200,300\}$, $m \in \{ 20,30,40,50\}$, and $s \in \{1,2,3\}$ to generate 10 instances for each $(n,m,s)$ triplet. This yields a total of 360 random instances.  \vspace{0.1cm}
\item[] \textsf{Dense data sets}: The expected total number of true arcs is controlled by $d$ and most of the instances have a dense true arc set compared with the sparse data sets. We use $n \in \{ 100,200,300\}$, $m \in \{ 20,30,40,50\}$, and $d \in \{0.1,0.2,0.3\}$ to generate 10 instances for each $(n,m,d)$ triplet, and thus we have a total of 360 random instances. 
\item[] \textsf{High dimensional data sets}: The instances are high dimensional ($m \geq n$) and very sparse. The expected total number of true arcs is controlled by $s$. We use $n = 100$ and $m \in \{100,150,200\}$, and $s \in \{0.5,1,1.5\}$ to generate 10 instances for each $(m,s)$ pair, which yields a total of 90 random instances.
\end{enumerate}

We use four $\lambda$ values differently defined for each data set in order to cover the expected number of arcs with the four $\lambda$ values. For each sparse instance, we solve \eqref{formualtion_regression_network} with $\lambda \in \{1,0.5,0.1,0.05\}$. For dense data sets, a wide range of $\lambda$ values are needed to obtain selected arc sets that have similar cardinalities with the true arc sets. Hence, for each dense instance, instead of fixed values over all instances in the set, we use $\lambda$ values based on expected density $d$: $\lambda = \lambda_0 \cdot 10^{-(10 \cdot d - 1)}$, where $\lambda_0 \in \{1,0.1,0.01,0.001\}$. For each high dimensional instance, we use $\lambda \in \{1,0.8,0.6,0.4\}$. Observe that the expected densities of the adjacency matrices vary across the three data sets. The sparse instances have expected densities between 0.02 and 0.15, the dense instances have expected densities between 0.1 and 0.3, and the high dimensional instances have expected densities between 0.002 and 0.015. Hence, different ranges of $\lambda$ values are necessary.

For all of the results presented in this section, we present the average performance by $n,m,s,d$ and $\lambda$. For example, the result for $n = 100$ and $m=20$ are the averages of 120 and 90 instances, respectively. In all of the comparisons, we use the following metrics.
\begin{enumerate}[noitemsep]
\item[] $time$: computation time in seconds
\item[] $\delta_{\mbox{\begin{tiny}sol\end{tiny}}}$: relative gap (\%) from the best objective value among the compared algorithms or models. For example, if we compare the three MIP models, $\delta_{\mbox{\begin{tiny}sol\end{tiny}}}$ of an MIP model is the relative gap from the best of the three objective function values obtained by the MIP models.
\item[] $\|z\|_0$: number of arcs selected (number of nonzero regression coefficients $\beta_{jk}$'s)
\end{enumerate}

In comparing the performance metrics, we use plot matrices. In each Figure \ref{fg_exp_sparse}, \ref{fg_exp_dense}, \ref{fg_exp_highdim}, \ref{fg_exp_sparse_mip}, and \ref{fg_exp_mip_vs_algo}, multiple bar plots form a matrix. The rows of the plot matrix correspond to performance metrics and the columns stand for parameters used for result aggregation. For example, the left top plot in Figure \ref{fg_exp_sparse} shows execution times of the algorithms where the results are aggregated by $n$ (the number of observations), because the first row and first column of the plot matrix in Figure \ref{fg_exp_sparse} are for execution times and $n$, respectively.

In Section \ref{section_exp_algo}, we compare the performance of iterative algorithms \textsf{GD10} and \textsf{IR10} and the benchmark algorithm \textsf{DIST}. In Section \ref{section_exp_mip}, we compare the performance of MIP models \textsf{MIPcp}, \textsf{MIPin}, and \textsf{MIPto}. We also compare all models and algorithms with a subset of the synthetic instances in Section \ref{section_exp_mip}. Finally, in Section \ref{subsec_real}, we solve a popular real instance of Sachs \textit{et al.} \cite{Sachs523} in the literature.

\subsection{Comparison of Iterative Algorithms}
\label{section_exp_algo}

In this section, we compare the performance of \textsf{GD10}, \textsf{IR10}, and \textsf{DIST} by $time$, $\delta_{\mbox{\begin{tiny}sol\end{tiny}}}$, and $\|z\|_0$ for each of the three data sets. 

In Figure \ref{fg_exp_sparse}, the result for the sparse data sets is presented. The bar plot matrix presents the performance measures aggregated by $n,m,s$, and $\lambda$.

\begin{figure}[ht]
\center
  \includegraphics[scale=0.38]{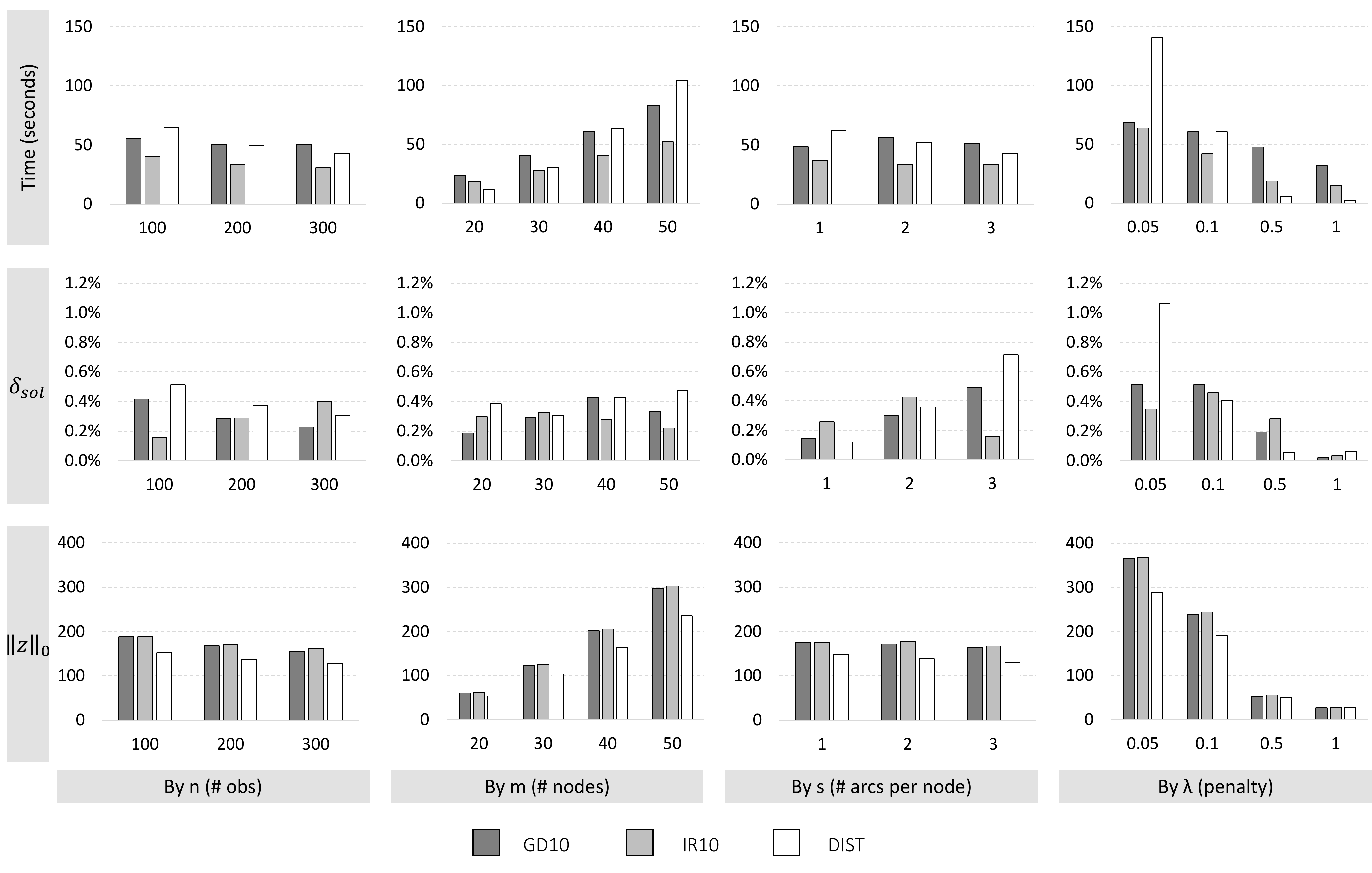}\\
  \caption{Performance of \textsf{GD10}, \textsf{IR10}, and \textsf{DIST} (sparse data)}
  \label{fg_exp_sparse}
\end{figure}

The computation time of all three algorithms increases in increasing $m$ and decreasing $\lambda$, where the computation time of \textsf{DIST} increases faster than the other two. The computation time of \textsf{DIST} is approximately 10 times faster than the \textsf{GD10} time when $\lambda = 1$, but 2 times slower when $\lambda = 0.05$. With increasing $n$, the computation times of all algorithms stay the same or decrease. This can be seen counter-intuitive because larger instances do not increase time. However, a larger number of observations can make predictions more accurate and could reduce search time for unattractive subsets. Especially, the computation time of \textsf{DIST} decreases in increasing $n$. We think this is because more observations give better local selection in the algorithm when adding and removing arcs. The number of selected arcs ($\|z\|_0$) of \textsf{GD10} and \textsf{IR10} is greater than \textsf{DIST} for all cases because the topological order based algorithms are capable of using the maximum number of arcs $\Big(\frac{m(m-1)}{2}\Big)$, while arc selection based algorithms, such as \textsf{DIST}, are struggling to select many arcs without violating acyclic constraints. In terms of the solution quality, all algorithms have $\delta_{\mbox{\begin{tiny}sol\end{tiny}}}$ less than 1.2\% and perform good. However, we observe several trends. As $\lambda$ decreases (required to select more arcs), \textsf{GD10} and \textsf{IR10} start to outperform. We also observe that, as the problem requires to select more arcs (increasing $m$, increasing $s$, and decreasing $\lambda$), \textsf{GD10} and \textsf{IR10} perform better. As $n$ increases, $\delta_{\mbox{\begin{tiny}sol\end{tiny}}}$ of \textsf{GD10} and \textsf{DIST} decrease, whereas $\delta_{\mbox{\begin{tiny}sol\end{tiny}}}$ of \textsf{IR10} increases.

The result for the dense data sets is presented in Figure \ref{fg_exp_dense}. The bar plot matrix presents the performance measures aggregated by $n,m,d,$ and $\lambda_0$. Recall that, for the dense data set, we solve \eqref{formualtion_regression_network} with $\lambda = \lambda_0 \cdot 10^{-(10 \cdot d - 1)}$ and $\lambda_0 \in \{1,0.1,0.01,0.001\}$. For simplicity of presenting the aggregated result, we use $\lambda_0$ in the plot matrix, while $\lambda$ values are used for actual computation. 

\begin{figure}[ht]
\center
  \includegraphics[scale=0.38]{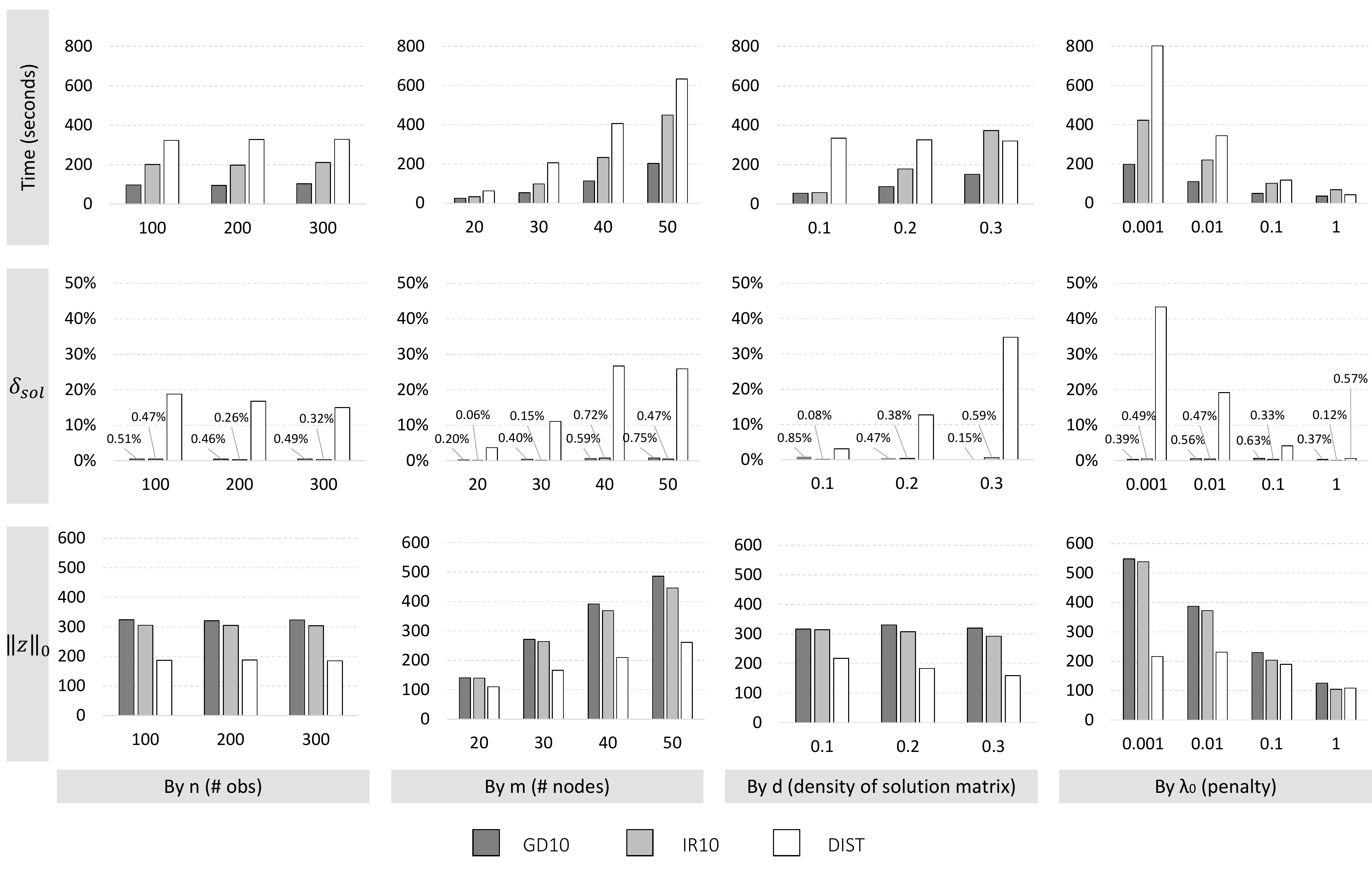}\\
  \caption{Performance of \textsf{GD10}, \textsf{IR10}, and \textsf{DIST} (dense data)}
  \label{fg_exp_dense}
\end{figure}

The computation time of all three algorithms again increases in increasing $m$ and decreasing $\lambda$, where the computation time of \textsf{DIST} increases faster than the other two. Compare to the result for the sparse data sets, the execution times are all larger for the dense data set. The number of selected arcs ($\|z\|_0$) of \textsf{GD10} and \textsf{IR10} is again greater than \textsf{DIST} for all cases, where $\|z\|_0$ is twice larger for \textsf{GD10} and \textsf{IR10} when $m$ or $d$ is large, or $\lambda$ is small. In terms of the solution quality, \textsf{GD10} and \textsf{IR10} outperform in most of the cases, while $\delta_{\mbox{\begin{tiny}sol\end{tiny}}}$ values of \textsf{DIST} increase fast in changing $m,d,$ and $\lambda$. The values of $\delta_{\mbox{\begin{tiny}sol\end{tiny}}}$ for all algorithms are larger than the sparse data sets result. \textsf{GD10} and \textsf{IR10} are better for most of the cases. In general, we again observe that \textsf{GD10} and \textsf{IR10} perform better when the problem requires to select more arcs.

The result for the high dimensional data sets is presented in Figure \ref{fg_exp_highdim}. The bar plot matrix presents the performance measures aggregated by $m,s,$ and $\lambda$, while $n$ is excluded from the matrix as we fixed $n$ to 100.

\begin{figure}[ht]
\center
  \includegraphics[scale=0.38]{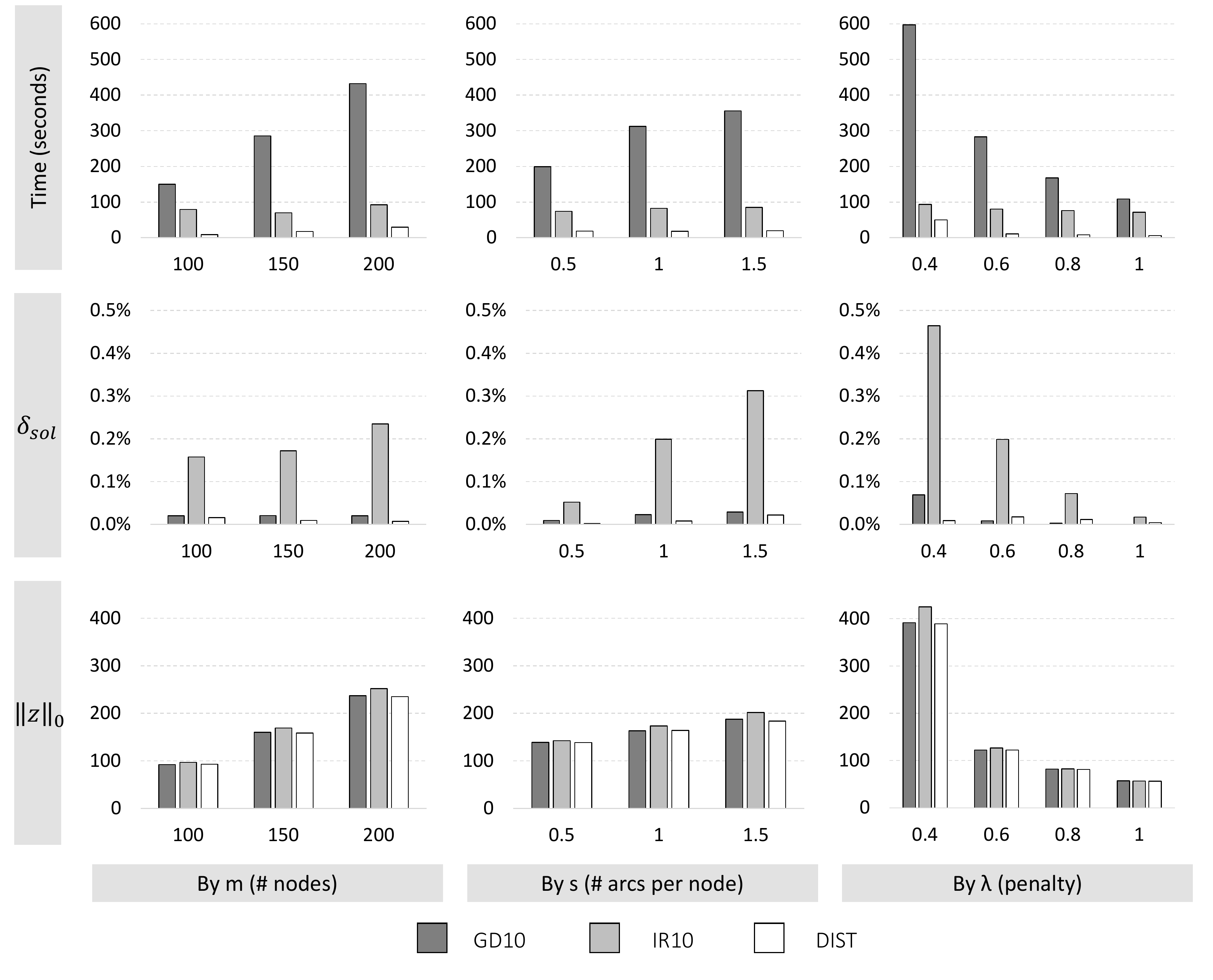}\\
  \caption{Performance of \textsf{GD10}, \textsf{IR10}, and \textsf{DIST} (high dimensional data)}
  \label{fg_exp_highdim}
\end{figure}

The computation time of all three algorithms again increases in increasing $m$ and decreasing $\lambda$. However, unlike the previous two sets, the computation times of \textsf{GD10} and \textsf{IR10} increase faster than \textsf{DIST}. This is due to the efficiency of the topological order based algorithms. When a very small portion of the arcs should be selected in the solution, topological orders are not informative. For example, consider a graph with three nodes $A,B,C$ and assume that only one arc, $(A,B)$, is selected due to a large penalty. For this case, three topological orders $A-B-C$, $A-C-B$, and $C-A-B$ can represent the selected arc. The third row of Figure \ref{fg_exp_highdim} shows that $\|z\|_0$ of the three algorithms are very similar, while \textsf{DIST} has much smaller values for the previous two data sets. This implies that arc based search by \textsf{DIST} does not have difficulties preventing cycles and the algorithm can decide whether to include arcs easier. The comparison of $\delta_{\mbox{\begin{tiny}sol\end{tiny}}}$ values also show that arc based search is competitive. Although all algorithms have $\delta_{\mbox{\begin{tiny}sol\end{tiny}}}$ values less than 0.5\%, we find clear evidence that the performance of \textsf{IR10}	decreases in increasing $m$ and $s$ and decreasing $\lambda$. Although the $\delta_{\mbox{\begin{tiny}sol\end{tiny}}}$ values of \textsf{GD10} and \textsf{DIST} are similar, considering the fast computing time of \textsf{DIST}, we recommend to use \textsf{DIST} for very sparse high dimensional data.

In Figure \ref{fg_density_vs_gaps}, we present combined results of all three data sets by relating solution densities and $\delta_{\mbox{\begin{tiny}sol\end{tiny}}}$. Observe that, for each quadruplet of $n,m,s$(or $d$)$,\lambda$, we have results from 10 random instances for each algorithm. Value $\delta_{\mbox{\begin{tiny}sol\end{tiny}}}$ is the average of the 10 results for each quadruplet and for each algorithm, $AvgDen$ is the average density of the adjacency matrices of the 10 results and the three algorithms. In Figure \ref{fg_density_vs_gaps}, we present a scatter plot of $ln(1 + AvgDen)$ and $ln(1+100 \cdot \delta_{\mbox{\begin{tiny}sol\end{tiny}}})$. Each point in the plot is the average of 10 results by an algorithm and each algorithm has 324 points displayed \footnote{$324 = (3 \cdot 4 \cdot 3 \cdot 4) + (3 \cdot 4 \cdot 3 \cdot 4) + (1 \cdot 3 \cdot 3 \cdot 4)$, where the three terms are for the three data sets and each term is obtained by multiplying the number of parameters $n,m,s(d)$, and $\lambda$, respectively.}. The numbers in the parenthesis along the axes are the corresponding values of $\delta_{\mbox{\begin{tiny}sol\end{tiny}}}$ and $AvgDen$. In the plot, we first observe that the algorithms perform similarly when the solutions are sparse and the $\delta_{\mbox{\begin{tiny}sol\end{tiny}}}$ values have large variance when the solutions are dense. When the log transformed solution density is less than 2, the average $\delta_{\mbox{\begin{tiny}sol\end{tiny}}}$ values of \textsf{GD10}, \textsf{IR10}, and \textsf{DIST} are 0.06\%, 0.15\%, and 0.04\%, respectively. However, the solution quality of \textsf{DIST} drastically decreases as the solutions become denser. This makes sense because sparse solutions can be efficiently searched by arc-based search, while dense solutions are not easy to obtain by adding or removing arcs one by one. This also explains the relatively small and large $\delta_{\mbox{\begin{tiny}sol\end{tiny}}}$ values for dense and spares solutions, respectively, by the topological order based algorithms. Between \textsf{GD10} and \textsf{IR10}, we do not observe a big difference.
\begin{figure}[ht]
\center
  \includegraphics[scale=0.6]{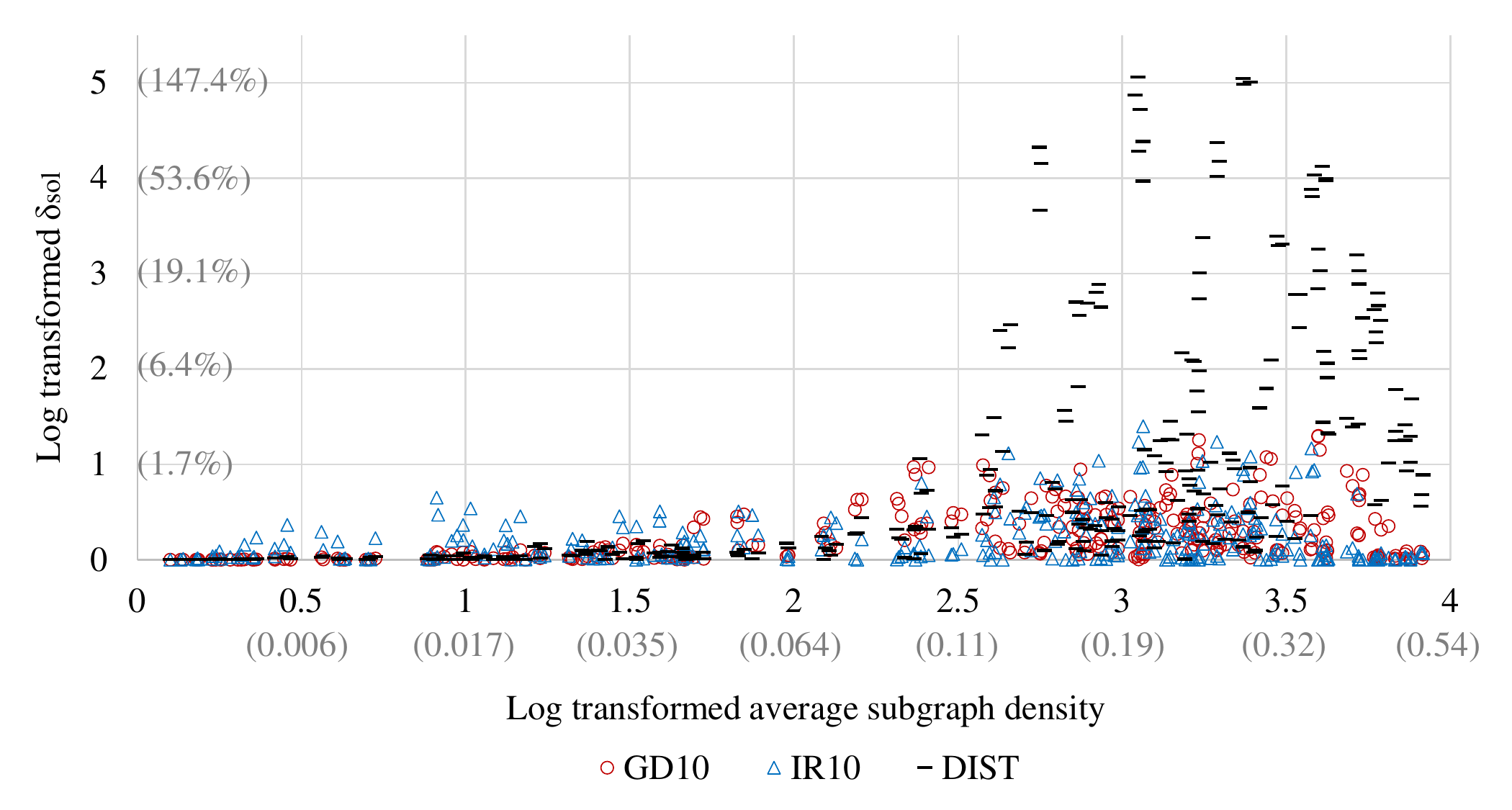}\\
      \caption{Scatter plot of $\delta_{sol}$ and average solution densities}
  \label{fg_density_vs_gaps}
\end{figure}

\subsection{Comparison of MIP Models}
\label{section_exp_mip}

In this section, we compare the performance of \textsf{MIPto}, \textsf{MIPin}, and \textsf{MIPcp} using $time$, $\delta_{\mbox{\begin{tiny}sol\end{tiny}}}$, and $\|z\|_0$ and the following additional metric.
\begin{enumerate}[noitemsep]
\item[] $\delta_{\mbox{\begin{tiny}IP\end{tiny}}}$: the optimality gap (\%) obtained by CPLEX within allowed $15 \cdot m$ seconds 
\end{enumerate}

Due to scalability issues of all models, we only use the sparse data with $m = 20,30,40$. We also limit $n = 100$. For all instances, we use the $15 \cdot m$ seconds time limit for CPLEX. For example, we have time limit of 300 seconds for instances with $m = 20$.

The result is presented in Figure \ref{fg_exp_sparse_mip}. Comparing the time of all models with the time limit for CPLEX, we observe that \textsf{MIPin} and \textsf{MIPcp} were able to terminate with optimality for several instances when $m = 20$ and $\lambda = 1$. This implies that \textsf{MIPin} and \textsf{MIPcp} are efficient when the problem is small and the number of selected arcs $\| z\|_0$ is small. However, in general, $\delta_{IP}$ values tend to be consistent with different models, while they increase in increasing $m$ and $s$ and in decreasing $\lambda$ for all three MIP models. The execution times of all models increase in increasing $m$ and $s$, and in decreasing $\lambda$. The same trend can be found for $\delta_{IP}$ for all models. By comparing $\delta_{\mbox{\begin{tiny}sol\end{tiny}}}$ values, we observe that  \textsf{MIPin} is best when $m = 20$ and 30. However, the performance of \textsf{MIPin} drops drastically as $m$ and $s$ increase and $\lambda$ decreases. Actually, \textsf{MIPin} fails to obtain a reasonably good solution within the time limit for several instances. This gives large $\delta_{\mbox{\begin{tiny}sol\end{tiny}}}$ values and increases the average. The $\delta_{\mbox{\begin{tiny}sol\end{tiny}}}$ values of \textsf{MIPto} are smaller than \textsf{MIPcp} when $m$ and $s$ are small, while \textsf{MIPcp} outperforms when $m = 40$ or $s = 3$.

\begin{figure}[ht]
\center
  \includegraphics[scale=0.4]{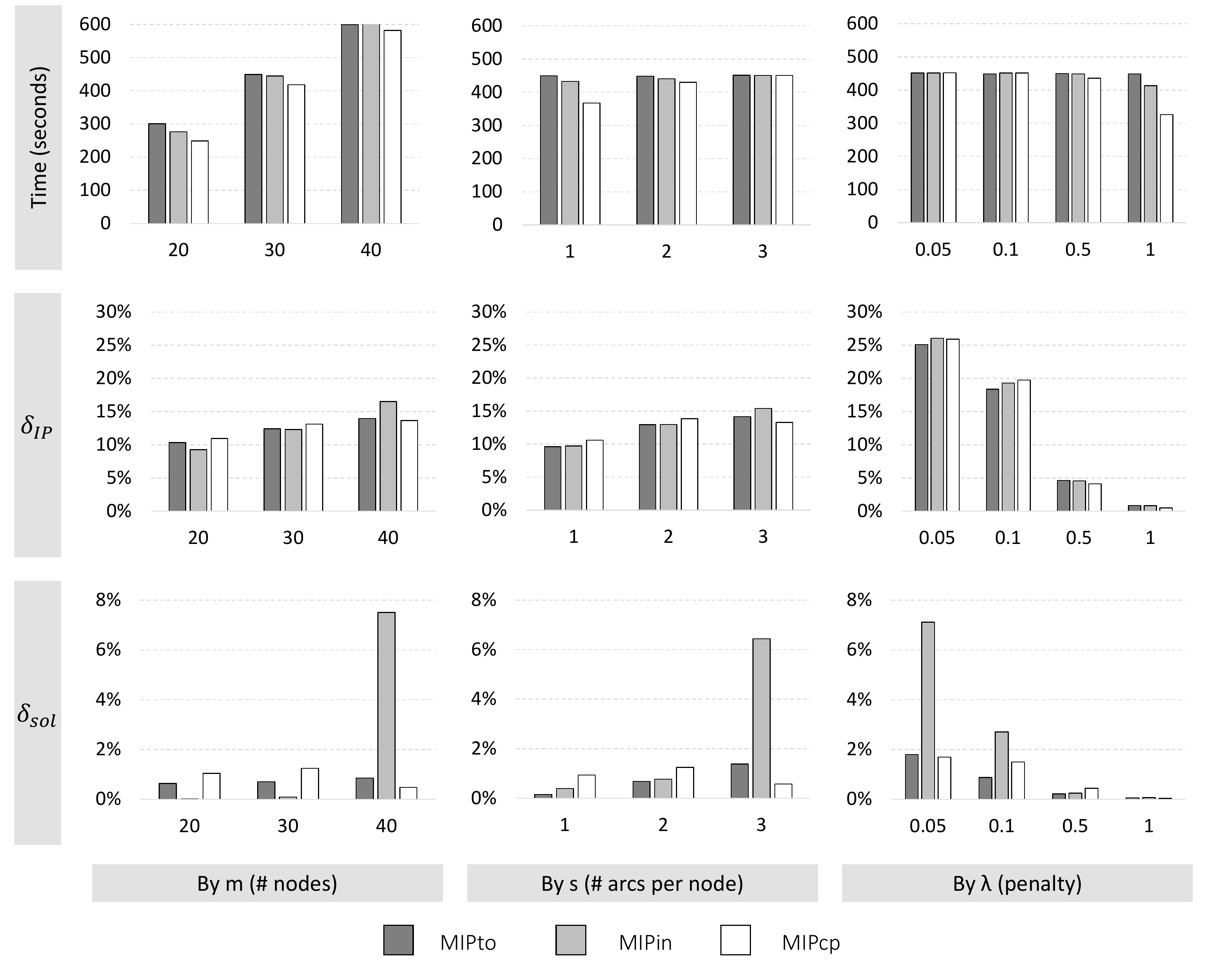}\\
      \caption{Performance of \textsf{MIPto}, \textsf{MIPin}, and \textsf{MIPcp} (sparse data with $n = 100$ and $m \in \{20,30,40\}$)}
  \label{fg_exp_sparse_mip}
\end{figure}

\subsection{Comparison of all MIP Models and Algorithms}
\label{section_exp_all}

In Figure \ref{fg_exp_mip_vs_algo}, we compare all models and algorithms for selected sparse instances with $n = 100$ and $m \in \{20,30,40\}$, which were used to test MIP models. In the plot matrix, we show the average computation time and $\delta_{\mbox{\begin{tiny}sol\end{tiny}}}$ (gap from the best objective value among the six models and algorithms) by $m,s,$ and $\lambda$. Note that $\delta_{\mbox{\begin{tiny}sol\end{tiny}}}$ values of a few \textsf{MIPin} results are not fully displayed in the bar plots due to their large values. Instead, the actual numbers are displayed next to the corresponding bar. The result shows that MIP models spent more time while the solution qualities are inferior in general. The values of $\delta_{\mbox{\begin{tiny}sol\end{tiny}}}$ for the MIP models are competitive only when $\lambda$ is large, which requires sparse solution. However, even for this case, MIP models spend longer time than the algorithms. Hence, ignoring the benefit of knowing and guaranteeing optimality by the MIP models, we conclude that the algorithms perform better for all cases.

\begin{figure}[ht]
\center
  \includegraphics[scale=0.38]{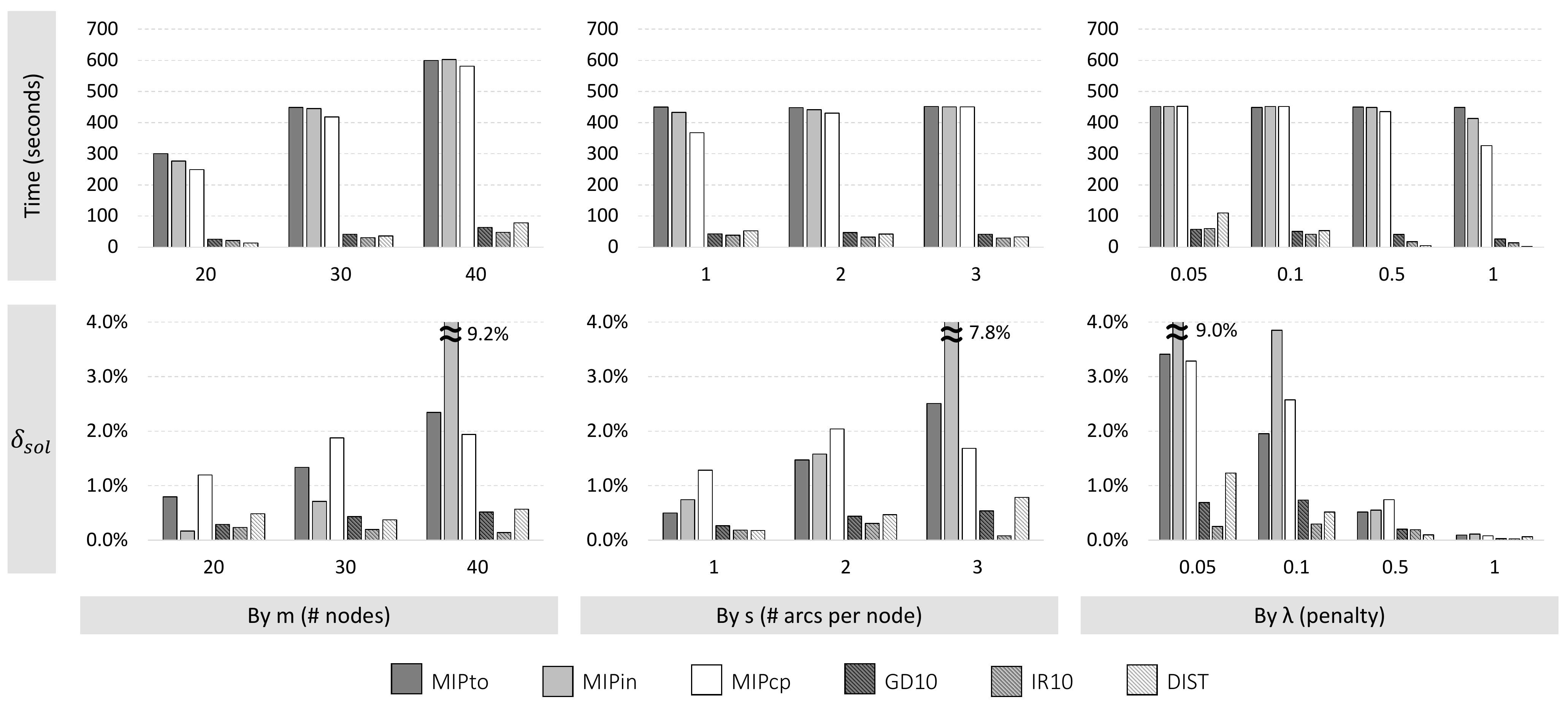}\\
      \caption{Performance of all models and algorithms (sparse data with $n = 100$ and $m \in \{20,30,40\}$)}
  \label{fg_exp_mip_vs_algo}
\end{figure}

The primary reason for inferior performance of the MIP models is the difficulty of solving integer programming problems. Further, all of the MIP models have at least $O(m^2)$ binary variables and $O(m^2)$ constraints and the problem complexity grows fast. Finally, large values of big M in \eqref{MIP_regression_network_supp} make the problem even more difficult. Due to non-tight values of big M, fathoming does not happen frequently in the branch and bound procedure. Hence, at least for sparse Gaussian Bayesian network learning, the MIP models may not be the best option unless other complicated constraints, which cannot be easily dealt with iterative algorithms, are needed.

\subsection{Real Data Example}
\label{subsec_real}

In this section, we study the flow cytometry data set from Sachs \textit{et al.} \cite{Sachs523} by solving \eqref{formualtion_regression_network}. The data set has been studied in many works including Friedman \textit{et al.} \cite{Friedman01072008}, Shojaie and Michailidis \cite{Shojaie01092010}, Fu and Zhou \cite{fu2013learning}, and Aragam and Zhou \cite{aragam2015concave}. The data set is often used as a benchmark as the casual relationships (underlying DAG) are known. It has $n = 7466$ cells obtained from multiple experiments with $m = 11$ measurements. The known structure contains 20 arcs. For the experiment, we standardize each column to have zero mean with standard deviation equal to one.

In Table \ref{table_exp_real}, we compare the performance of the three algorithms \textsf{GD10}, \textsf{IR10}, and \textsf{DIST} for various values of $\lambda$. The MIP models are excluded due to scalability issue \footnote{Although $m$ is small, with $n=$7466 observations, the MIP models have at least $7466 \cdot 11 \cdot 2$ = 164,252 continuous variables and 7466 $\cdot$ 11 = 82,126 constraints just for the residual terms. Combined with the complexity increment due to the binary variables and acyclicity constraints, it was not feasible to obtain a reasonable solution by any of the MIP models.}. We compare the previously used performance measures execution time, solution cardinalities ($\|z\|_0$), and solution quality ($\delta_{\mbox{\begin{tiny}sol\end{tiny}}}$). In addition, we also compare sensitivities (true positive ratio) of the solutions. By comparing with the known structure with 20 arcs, we calculate directed true positive (dTP) and undirected true positive (uTP). If arc $(j,k)$ is in the known structure, dTP counts only if arc ($j,k$) is in the algorithm's solution, whereas uTP counts either of arcs $(j,k)$ or $(k,j)$ is in the algorithm's solution. 

The solution times of the three algorithms are all within a few seconds, as $m$ is small. Also, the solution cardinalities ($\|z\|_0$) are similar. The best $\delta_{\mbox{\begin{tiny}sol\end{tiny}}}$ value among the three algorithms in each row is in boldface. We observe that \textsf{GD10} provides the best solution (smallest $\delta_{\mbox{\begin{tiny}sol\end{tiny}}}$) in most cases and \textsf{IR10} is the second best. The $\delta_{\mbox{\begin{tiny}sol\end{tiny}}}$ values of \textsf{DIST} increase as $\lambda$ increases. This is consistent with the findings in Section \ref{section_exp_algo}. Note that the density of the underlying structure is $20/11^2 = 0.165$, which is dense. This explains the good performance of \textsf{GD10} and \textsf{IR10}. On the other hand, even though \textsf{GD10} provides the best objective function values for most of the cases, the dTP and uTP values of \textsf{GD10} are not always the best. The highest value among the three algorithms in each row is in boldface. While $\delta_{\mbox{\begin{tiny}sol\end{tiny}}}$ values of \textsf{DIST} are the largest among the three algorithms, dTP and uTP values are the best in some cases. When $\lambda$ is small, \textsf{DIST} tends to have higher dTP and uTP. However, as $\lambda$ increases, \textsf{GD10} gives the best dTP and uTP values. To further improve the prediction power, we may need weighting features or observations.
\begin{table}[htbp]
  \centering
\scriptsize
\setlength{\tabcolsep}{5pt}
    \begin{tabular}{|c|rrrrr|rrrrr|rrrrr|}
    \hline
          & \multicolumn{5}{c|}{\textsf{GD10}}      & \multicolumn{5}{c|}{\textsf{IR10} }     & \multicolumn{5}{c|}{\textsf{DIST} } \\ \hline
    $\lambda$   & $time$  & $\|z\|_0$     & $\delta_{\mbox{\begin{tiny}sol\end{tiny}}}$ & dTP & uTP     & $time$  & $\|z\|_0$    & $\delta_{\mbox{\begin{tiny}sol\end{tiny}}}$  & dTP & uTP & $time$  & $\|z\|_0$     & $\delta_{\mbox{\begin{tiny}sol\end{tiny}}}$ & dTP & uTP \\ \hline
    0.5   & 14.6  & 9     & \textbf{0.00\%} & 0.22  & \textbf{0.56}  & 10.8  & 9     & 0.22\% & 0.22  & \textbf{0.56}  & 5.8   & 9     & 0.38\% & \textbf{0.56}  & \textbf{0.56} \\
    0.45  & 13.7  & 11    & \textbf{0.00\%} & 0.27  & 0.55  & 10.9  & 13    & 0.30\% & 0.15  & 0.46  & 8.5   & 9     & 0.49\% & \textbf{0.56}  & \textbf{0.56} \\
    0.4   & 14.3  & 11    & \textbf{0.00\%} & 0.27  & \textbf{0.55}  & 14.8  & 13    & 0.36\% & 0.15  & 0.46  & 7.0   & 13    & 0.61\% & \textbf{0.38}  & 0.46 \\
    0.35  & 16.0  & 13    & \textbf{0.00\%} & 0.23  & \textbf{0.54}  & 12.4  & 17    & 0.42\% & 0.12  & 0.41  & 8.6   & 15    & 0.73\% & \textbf{0.33}  & 0.47 \\
    0.3   & 15.1  & 13    & \textbf{0.00\%} & 0.23  & 0.46  & 14.3  & 17    & 0.21\% & 0.24  & 0.41  & 8.0   & 17    & 0.84\% & \textbf{0.29}  & \textbf{0.47} \\
    0.25  & 14.4  & 16    & \textbf{0.00\%} & \textbf{0.38}  & \textbf{0.56}  & 14.8  & 20    & 0.22\% & 0.25  & 0.50  & 6.7   & 20    & 0.94\% & 0.30  & 0.50 \\
    0.2   & 16.0  & 16    & \textbf{0.00\%} & 0.31  & \textbf{0.56}  & 14.6  & 22    & 0.31\% & 0.32  & 0.55  & 6.8   & 21    & 1.12\% & \textbf{0.33}  & 0.52 \\
    0.15  & 16.4  & 21    & \textbf{0.00\%} & \textbf{0.33}  & \textbf{0.57}  & 15.6  & 25    & 0.32\% & 0.32  & 0.52  & 9.5   & 23    & 1.26\% & 0.30  & 0.52 \\
    0.1   & 17.4  & 24    & \textbf{0.00\%} & 0.25  & \textbf{0.50}  & 11.0  & 28    & 0.26\% & 0.07  & 0.43  & 9.1   & 27    & 1.45\% & \textbf{0.26}  & 0.44 \\
    0.05  & 17.7  & 28    & 0.32\% & \textbf{0.39}  & \textbf{0.50}  & 11.5  & 30    & \textbf{0.00\%} & 0.10  & 0.47  & 10.1  & 32    & 1.87\% & 0.22  & 0.41 \\ \hline
    \end{tabular}%
      \caption{Performance on the real data set from Sachs \textit{et al.} \cite{Sachs523}}
  \label{table_exp_real}%
\end{table}%

From Table \ref{table_exp_real}, we observe that  a slight change in solution quality $\delta_{\mbox{\begin{tiny}sol\end{tiny}}}$ affects the final selection of DAG significantly. Also, the best objective function value does not necessarily give the highest true positive value since the $L_1$-norm penalized least square \eqref{formualtion_regression_network} may not be the best score function. In Figure \ref{fg_sachs}, we present the graphs of known casual interactions (a), estimated subgraph by \textsf{GD10} (b), \textsf{IR10} (c), and \textsf{DIST} (d). All graphs are obtained with $\lambda = 0.25$, but the numbers of arcs are different (see Table \ref{table_exp_real}). In fact, the difference between $\delta_{\mbox{\begin{tiny}sol\end{tiny}}}$ values of \textsf{IR10} and \textsf{DIST} is less than 1\% while the subgraphs have only 10 common arcs. 
\begin{figure}[ht]
\center
  \includegraphics[scale=0.35]{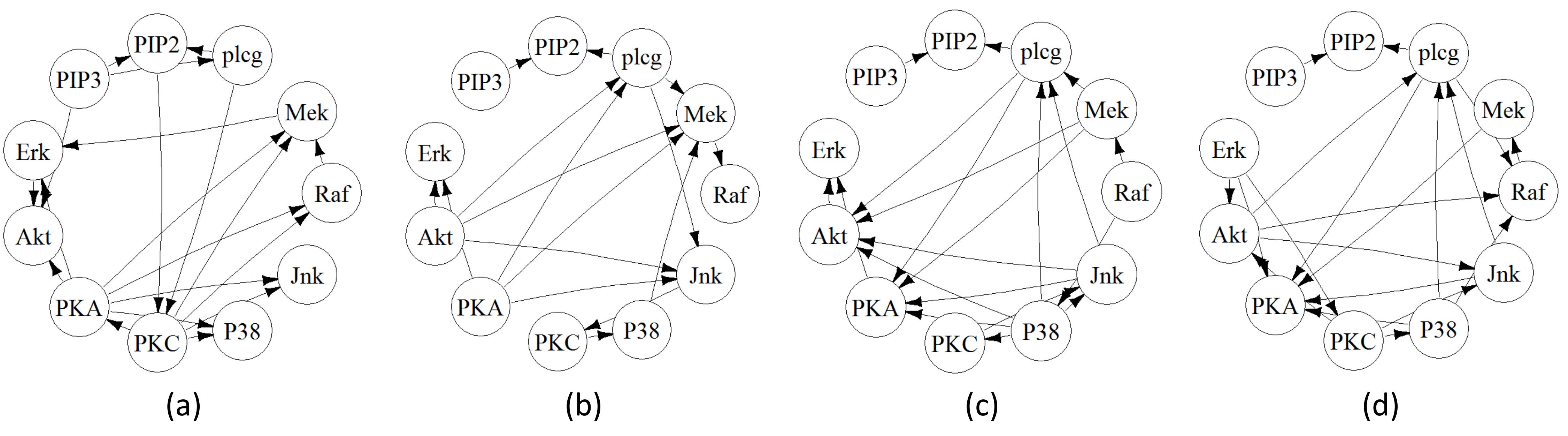}\\
      \caption{Known DAG (a) and estimated subgraphs with $\lambda = 0.25$ by \textsf{GD10} (b), \textsf{IR10} (c), and \textsf{DIST} (d)}
  \label{fg_sachs}
\end{figure}

\section{Conclusion}

We propose an MIP model and iterative algorithms based on topological order. Although the computational experiment is conducted for Gaussian Bayesian network learning, all the proposed model and algorithms are applicable for problems following the form in \eqref{def_opt_with_dag}. While many MIP models and algorithms are designed based on arc search, using topological order provides some advantages that improve solution quality and algorithm efficiency.

\begin{enumerate}
\item DAG constraints (acyclicity constraints) are automatically satisfied when arcs from high order nodes to low order nodes are used. 
\item In applying the concept for MIP, a lower number of constraints is needed $(O(m^2))$ whereas arc based modeling can have exponentially many constraints in the worst case.
\item In applying the concept in designing iterative algorithms, one of the biggest merits is the capability of utilizing the maximum number of arcs possible ($\frac{m(m-1)}{2}$), while arc based algorithms struggle with using all possible arcs.
\end{enumerate}

The proposed MIP model has the smallest number of constraints while the number of binary variables is in the same order with the already known MIP models. It performs as good as a cutting plane algorithm. The proposed iterative algorithms get the biggest benefit when the solution matrix is dense. The result presented in Section \ref{section_exp_algo} clearly indicates that the topological order based algorithms outperform when the density of the resulting solution is high. On the other hand, arc-based search algorithms, represented by DIST in our experiment, can be efficient when the desired solutions are very sparse.

Comparing all models and algorithms used in the experiment, we observe that the MIP models are not competitive or scale well compared to the heuristic algorithms except for small instances. The experiment shows that the solution times of the MIP models are significantly affected by the number of nodes $m$. For Gaussian Bayesian network learning, we observe that large $n$ could also decrease the MIP model efficiency even when $m$ is small (Section \ref{subsec_real}). Among the iterative algorithms, we recommend \textsf{DIST} for very sparse high dimensional data and \textsf{GD10} and \textsf{IR10} for dense data. Among the two topological order based algorithms, \textsf{GD10} performs slightly better and is more stable.

\bibliographystyle{abbrv}


\appendix
\section*{APPENDIX}

\section{Greedy Algorithm for Projection Problem}
\label{appendix_greedy_algo}

In this section, we present the detail derivations and proofs of Algorithm \ref{algo_greedy} (\textit{greedy}). The algorithm sequentially determines topological order by optimizing the projection problem given an already fixed order up to the iteration point. Solving the projection problem is $\mathcal{NP}$-complete and our algorithm may not give a global optimal solution. However, the result of this section shows that Algorithm \ref{algo_greedy} gives an optimal choice of the next node to have fixed order given pre-fixed orders.

We start by describing some properties of $Y^*$ in the following three lemmas. For the following lemmas, let $\pi_j^*$ represent the topological order of node $j$ defined by $Y^*$.
\begin{lemma}
\label{lemma_zero_or_u}
For any $Y^*_{jk}$, we must have either $Y_{jk}^* = 0$ or $Y_{jk}^* = U_{jk}^t$.
\end{lemma}
\begin{proof}
For a contradiction, let us assume that there exist indices $q$ and $r$ such that $Y_{qr}^* \neq 0$ and $Y_{qr}^* \neq U_{qr}^t$. Let us create a new solution $\bar{Y}$ such that $\bar{Y} = Y^*$ except $\bar{Y}_{qr} = U_{qr}^t$. Note that $\bar{Y}$ is a feasible solution to \eqref{def_opt_unconstrained_acyclic} because $supp(\bar{Y}) \leq supp(Y^*)$ element-wise since $Y_{qr}^* \neq 0$. Further, we have $\| Y^* - U^t\|_2 > \| \bar{Y} - U^t\|_2$ because $(Y_{qr}^*- U_{qr}^t)^2 > 0 = (\bar{Y}_{qr} - U_{qr}^t)^2$ and $\bar{Y} = Y^*$ except $Y_{qr}^* \neq \bar{Y}_{qr}$. This contradicts optimality of $Y^*$.
\end{proof}
Note that Lemma \ref{lemma_zero_or_u} implies that solving \eqref{def_opt_unconstrained_acyclic} is essentially choosing between 0 and $U_{jk}^*$ for $Y_{jk}^*$. This selection is also based on the following property.

\begin{lemma}
\label{lemma_toporder_u}
If $\pi_j^* > \pi_k^*$, then $Y_{jk}^* = U_{jk}^t$.
\end{lemma}
\begin{proof}
For a contradiction, let us assume that there exist indices $q$ and $r$ such that $Y_{qr}^* \neq U_{qr}^t$ while $\pi_q^* > \pi_r^*$. Let us create a new solution $\bar{Y}$ such that $\bar{Y} = Y^*$ except $\bar{Y}_{qr} = U_{qr}^t$. 
\begin{enumerate}
\item If $Y_{qr}^* \neq 0$, then $\bar{Y}$ is a DAG since $supp(\bar{Y}) \leq supp(Y^*)$ element-wise.
\item If $Y_{qr}^* = 0$, then arc $(q,r)$ can be used in the solution without creating a cycle because $\pi_q^* > \pi_r^*$. Hence, $\bar{Y}$ is a DAG.
\end{enumerate}
Therefore, $\bar{Y}$ is a feasible solution to \eqref{def_opt_unconstrained_acyclic}. However, it is easy to see that $\| \bar{Y} - U^t\|_2 < \| Y^* - U^* \|_2$ because $( Y_{qr}^* - U_{qr}^t)^2 > 0 = ( \bar{Y}_{qr} - U_{qr}^t)^2$. This contradicts optimality of $Y^*$.
\end{proof}
Given the topological order by $Y^*$, let $\hat{J}^k = \{ j \in J^k | \pi_j^* > \pi_k^* \}$ be the subset of $J^k$ such that the nodes in $\hat{J}^k$ are earlier than node $k$. Combining Lemmas \ref{lemma_zero_or_u} and \ref{lemma_toporder_u}, we conclude that $Y^*$ has the following structure.
\begin{equation}
\label{def_opt_y_structure}
\begin{tabular}{ll}
$Y_{jk}^*= \left \{
	\begin{array}{ll}
		U_{jk}^{t}  & \mbox{ if } j \in \hat{J}^k,  \\
		0 &  \mbox{ if } j \in J \setminus \hat{J}^k, 
	\end{array}
\right.$ & $k \in J$
\end{tabular}
\end{equation}
Further, we can calculate node $k$'s contribution to the objective function value without explicitly using $Y^*$.
\begin{lemma}
\label{lemma_contribution_to_obj}
For each node $k \in J$, it contributes $$\sum_{j \in J \setminus \hat{J}^k} ( U_{jk}^t )^2$$ to the objective function value for \eqref{def_opt_unconstrained_acyclic}. In other words, the contribution of node $k$ is the squared sum of $U_{jk}^t$ for nodes with $\pi_j^* < \pi_k^*$.
\end{lemma}
\begin{proof}
For node $k$, we can derive
\begin{center}
$\sum_{j \in J} ( Y_{jk}^* - U_{jk}^t )^2 = \sum_{j \in J \setminus \hat{J}^k} ( Y_{jk}^* - U_{jk}^t )^2 = \sum_{j \in J \setminus \hat{J}^k} ( U_{jk}^t )^2$,
\end{center}
where both equal signs are due to \eqref{def_opt_y_structure}. The first equality is due to $Y_{jk}^* = U_{jk}^{t}$ for $j \in \hat{J}^k$ and the second equality holds since $Y_{jk}^* = 0$ for $j \in J \setminus \hat{J}^k$.  
\end{proof}

We next detail the derivation of the greedy algorithm presented in Algorithm \ref{algo_greedy}. Let $\bar{J} \subseteq J$ be the index set of yet-to-be-ordered nodes and $\bar{J}^c = J \setminus \bar{J}$ be the index set of the nodes that have already been ordered. The procedure is equivalent to iteratively solving
\begin{equation}
\label{def_greedy_prob}
\displaystyle k^* = \mbox{argmin}_{k \in \bar{J}} \Big\{ \min_{Y_k} \big\{ \sum_{j \in J} ( Y_{jk} - U_{jk}^t)^2 \big\} \Big\},
\end{equation}
where $Y_k = [Y_{1k},Y_{2k},\cdots,Y_{mk}] \in \mathbb{R}^{m \times 1}$ is the column in $Y$ corresponding to node $k$. Set $\bar{J}$ is updated by $\bar{J} = \bar{J} \setminus \{k^*\}$ and $\pi_{k^*}^{*} = |\bar{J}|$ after solving \eqref{def_greedy_prob}. We propose an algorithm to solve \eqref{def_greedy_prob} based on the properties of $Y^*$ described in Lemmas \ref{lemma_zero_or_u} - \ref{lemma_contribution_to_obj}. Given $\bar{J}$, we solve
\begin{equation}
\label{def_greedy_prob_algorithm}
k^* = \mbox{argmin}_{k \in \bar{J}} \Big\{ \sum_{j \in \bar{J}} (U_{jk}^t)^2 \Big\}.
\end{equation}
Next, we show that solving \eqref{def_greedy_prob_algorithm} gives an optimal solution $\bar{Y}_{jk^*}$, $j \in J^k$ to \eqref{def_greedy_prob}. We can actually replicate the properties of $Y^*$ for $\bar{Y}_{jk^*}$.
\begin{lemma}
An optimal solution to \eqref{def_greedy_prob} must have either $\bar{Y}_{jk^*} = 0$ or $\bar{Y}_{jk^*} = U_{jk^*}^t$ for all $j \in J^k$.
\end{lemma}
\begin{lemma}
An optimal solution to \eqref{def_greedy_prob} must have $\bar{Y}_{jk^*} = U_{jk^*}^t$ for $j \in J \setminus \bar{J}$.
\end{lemma}
\begin{lemma}
\label{lemma_opt_greedy}
An optimal solution to \eqref{def_greedy_prob} must satisfy $\sum_{j \in J} ( \bar{Y}_{jk^*} - U_{jk^*}^t)^2 = \sum_{j \in \bar{J}} (U_{jk^*}^t)^2$
\end{lemma}
The proofs are omitted as they are similar to the proofs of Lemmas \ref{lemma_zero_or_u} - \ref{lemma_contribution_to_obj}, respectively. Note that, by Lemma \ref{lemma_opt_greedy}, we show the equivalence of $\sum_{j \in J} ( Y_{jk^*} - U_{jk^*}^t)^2$ and $\sum_{j \in \bar{J}} (U_{jk^*}^t)^2$. This result also holds for each term, $k \in \bar{J}$, of the argmin function in \eqref{def_greedy_prob}. Hence, the following lemma holds.
\begin{lemma}
Solving \eqref{def_greedy_prob_algorithm} is equivalent to solving \eqref{def_greedy_prob}. 
\end{lemma}
Observe that \eqref{def_greedy_prob_algorithm} is used in Line 3 of the greedy algorithm in Algorithm \ref{algo_greedy}. Hence, by property \eqref{def_greedy_prob}, the algorithm gives an optimal choice of node to be fixed with the pre-fixed topological order.

\section{Summary Statistics for Maximum Coefficients}
\label{Appendix_max_coef}

In this section, we show that the heuristic formula \eqref{eqn_bigM} for selecting big M gives reasonable and large enough values. Note that, when creating a synthetic instance, we used a random DAG to generate multivariate data. Although the optimal DAG for the penalized least squares are unknown, with appropriate penalty constants, we can use the implanted DAG to obtain coefficients estimation. That is, we calculate $$ \hat{B} = \max_{j \in J^k, k \in K} |\hat{\beta}_{jk}|$$ by using the implanted DAG. The $\hat{B}$ value is then compared with the big $M$ value in \eqref{eqn_bigM} for all instances used for MIP models (sparse data with $n=100$ and $m \in \{20,30,40\}$). We calculate minimum, average, and maximum of $\hat{B}$ and $M$ by $m,s,$ and $\lambda$. The result is presented in Table \ref{tab_max_coef}. Note that we have $\hat{B} < M$ for all columns Min, Avg, and Max and for all rows. Although only the summary statistics are presented in Table \ref{tab_max_coef}, we observed that $M$ is greater than $\hat{B}$ for all cases considered.

\begin{table}[htbp]
  \centering
  \small
    \begin{tabular}{|r|r|rrr|rrr|}
    \hline
\multicolumn{2}{|c|}{By}       & \multicolumn{3}{c|}{$\hat{B}$} & \multicolumn{3}{c|}{$M$} \\ \hline
  Param        &  Value     & Min   & Avg   & Max   & Min   & Avg   & Max \\ \hline
    \multicolumn{1}{|c|}{\multirow{3}[0]{*}{m}} & 20    & 0.17  & 0.65  & 1.48  & 0.33  & 1.34  & 3.73 \\
    \multicolumn{1}{|c|}{} & 30    & 0.18  & 0.62  & 0.95  & 0.36  & 1.32  & 3.10 \\
    \multicolumn{1}{|c|}{} & 40    & 0.18  & 0.63  & 1.03  & 0.36  & 1.33  & 3.30 \\ \hline
    \multicolumn{1}{|c|}{\multirow{3}[0]{*}{s}} & 1     & 0.17  & 0.53  & 0.79  & 0.33  & 1.13  & 2.04 \\ 
    \multicolumn{1}{|c|}{} & 2     & 0.23  & 0.63  & 0.94  & 0.45  & 1.33  & 3.54 \\
    \multicolumn{1}{|c|}{} & 3     & 0.32  & 0.74  & 1.48  & 0.63  & 1.53  & 3.73 \\ \hline
    \multicolumn{1}{|c|}{\multirow{4}[0]{*}{$\lambda$}} & 0.05  & 0.62  & 0.82  & 1.48  & 1.24  & 1.95  & 3.73 \\ 
    \multicolumn{1}{|c|}{} & 0.1   & 0.60  & 0.78  & 0.96  & 1.19  & 1.54  & 2.35 \\
    \multicolumn{1}{|c|}{} & 0.5   & 0.41  & 0.59  & 0.74  & 0.82  & 1.14  & 1.45 \\
    \multicolumn{1}{|c|}{} & 1     & 0.17  & 0.35  & 0.49  & 0.33  & 0.69  & 0.97 \\ \hline
    \end{tabular}%
      \caption{Comparison of maximum coefficients and big M (sparse data with $n = 100$ and $m \in \{20,30,40\}$)}
  \label{tab_max_coef}%
\end{table}%

\end{document}